\documentclass{article} 
\usepackage{iclr2025_conference,times, amsthm}

\usepackage{amsmath,amsfonts,bm}

\def\eqref#1{equation~\ref{#1}}

\def\1{\bm{1}}

\DeclareMathAlphabet{\mathsfit}{\encodingdefault}{\sfdefault}{m}{sl}
\SetMathAlphabet{\mathsfit}{bold}{\encodingdefault}{\sfdefault}{bx}{n}

\newcommand{\R}{\mathbb{R}}

\newcommand{\softmax}{\mathrm{softmax}}

\newtheorem{theorem}{Theorem}[section]
\newtheorem{corollary}[theorem]{Corollary}

\newtheorem{lemma}[theorem]{Lemma}
\newtheorem{definition}[theorem]{Definition}
\newtheorem{conjecture}[theorem]{Conjecture}

\DeclareMathOperator{\poly}{poly}
\DeclareMathOperator{\polylog}{polylog}

\usepackage{hyperref}
\usepackage{url,cleveref,amssymb}

\title{Fundamental Limitations on Subquadratic Alternatives to Transformers}

\author{Josh Alman\thanks{* denotes equal contribution} \\
  Department of Computer Science \\
  Columbia University \\
  New York, NY 10027, USA \\
  \texttt{josh@cs.columbia.edu} 
  \And
  Hantao Yu\footnotemark[1] \\
  Department of Computer Science \\
  Columbia University \\
  New York, NY 10027, USA \\
  \texttt{hantao.yu@columbia.edu}
}

\usepackage{enumitem,url}
\newcommand{\inn}{\textup{in}}
\newcommand{\out}{\textup{out}}
\newcommand{\SETH}{\mathsf{SETH}}
\newcommand{\OV}{\mathsf{OV}}
\newcommand{\OVC}{\mathsf{OVC}}
\newcommand{\MSD}{\mathsf{MSD}}
\newcommand{\LSD}{\mathsf{LSD}}
\newcommand{\SAT}{\mathsf{SAT}}
\newcommand{\PP}{\mathsf{P}}
\newcommand{\NP}{\mathsf{NP}}
\newcommand{\MinIP}{\mathsf{Min}\text{-}\mathsf{IP}}
\newcommand{\MaxIP}{\mathsf{Max}\text{-}\mathsf{IP}}
\newcommand{\BMaxIP}{\mathsf{BMax}\text{-}\mathsf{IP}}

\newcommand{\BOW}{\mathsf{BOW}}
\newcommand{\TF}{\mathsf{TF}}

\iclrfinalcopy
\begin{document}

\maketitle

\begin{abstract}
    The Transformer architecture is widely deployed in many popular and impactful Large Language Models. At its core is the attention mechanism for calculating correlations between pairs of tokens. Performing an attention computation takes quadratic time in the input size, and had become the time bottleneck for transformer operations. In order to circumvent this, researchers have used a variety of approaches, including designing heuristic algorithms for performing attention computations faster, and proposing alternatives to the attention mechanism which can be computed more quickly. For instance, state space models  such as Mamba were designed to replace attention with an almost linear time alternative.

    In this paper, we prove that any such approach cannot perform important tasks that Transformer is able to perform (assuming a popular conjecture from fine-grained complexity theory). We focus on document similarity tasks, where one is given as input many documents and would like to find a pair which is (approximately) the most similar. We prove that Transformer is able to perform this task, and we prove that this task cannot be performed in truly subquadratic time by any algorithm. Thus, any model which can be evaluated in subquadratic time -- whether because of subquadratic-time heuristics for attention, faster attention replacements like Mamba, or any other reason -- cannot perform this task. In other words, in order to perform tasks that (implicitly or explicitly) involve document similarity, one may as well use Transformer and cannot avoid its quadratic running time.
\end{abstract}

\section{Introduction}












The Transformer architecture \citep{VSPUJGKP17} is widely used for natural language processing \citep{DCLT19,YDYCSL19}, computer vision \citep{DBK+21,CMSUK20}, and many other tasks, and has achieved state-of-the-art performance for numerous applications. At the core of the architecture is the attention mechanism which is designed to calculate the correlation between all pairs of tokens in a given input sequence. Namely, let $Q\in \mathbb{R}^{d_{\inn} \times m}$ be the query matrix, $K \in \mathbb{R}^{d_{\inn}\times m}$ be the key matrix and $V \in \mathbb{R}^{d_{\inn} \times d_{\out}}$ be the value matrix. Given $X \in \mathbb{R}^{n\times d_{\inn}}$, an attention mechanism computes 
\[
A_{Q,K,V}(X) := \softmax(XQK^{\top}X^{\top})XV
\] where the softmax operator 
\[
\softmax(v) = \frac{(\exp(v[1]),\ldots,\exp(v[n]))}{\sum_{i=1}^{n}\exp(v[i])}
\] for $v \in \mathbb{R}^n$ is applied to matrices row-wise. Computing the attention by straightforwardly following the definition above requires quadratic time (in the sequence length $n$), which prohibits efficient model training when the sequence length is too large. As a result, much effort has been devoted to overcoming this obstacle in recent years, and there are two major lines of research to tackle the problem.

The first line of research argues that instead of computing attention exactly in the worst case, it often suffices to use heuristics which work well when the input data has additional structure, or to only return coarse approximations of the attention mechanism which can be computed more quickly. Examples including KDEformer \citep{ZHDK23}, Reformer \citep{KKL20}, Hyperattention \citep{HJKMW24}, Linformer \citep{WLKFM20}, SMYRF \citep{DKOD20}, and Performer \citep{CLD+21}. In many cases, these heuristics result in algorithms which run in close to linear time. However, these techniques usually have corresponding downsides, such as model accuracy drops, or performance gains which do not appear to scale to large inputs.

The second line of research argues that instead of computing or approximating attention, we can replace the standard attention mechanism with new, different mechanisms which can be computed faster. Models such as Longformer \citep{BPC20}, Synthesizer \citep{Synthesizer21}, Routing transformers \citep{routingtransformer21}, and MAMBA \citep{MAMBA23} all aim to circumvent the quadratic barrier by proposing new attention alternatives. A priori, these techniques would result in weaker expressiveness since they replace attention's calculation of token interactions with simpler alternatives, although most also provide empirical evidence that the loss in accuracy at certain tasks is not large.

In this paper, we prove that any approach that takes subquadratic time, no matter whether it uses heuristic or approximations, or a new architecture, or a completely different approach, is inherently unable to perform important learning tasks that a transformer is able to perform. By using tools and popular hardness conjectures from the area of fine-grained complexity theory, we show that many learning tasks involving document similarity cannot possibly be solved in subquadratic time using \emph{any} algorithmic approach. This implies that subquadratic alternatives to standard transformers are not able to solve these simple and natural tasks in machine learning and NLP. To complement this, we show that standard transformers (even simple transformers with one layer and one attention head) are able to perform these tasks, thereby establishing a separation between standard transformers and these new approaches. In other words, we prove that accuracy loss for any task relating to document similarity is \emph{unavoidable} for any subquadratic approach, even when compared only to the simplest transformers, because of the inherent computational complexity of the task.

\subsection{Document Similarity}


In this work we will be focusing on document similarity tasks. We will show that standard transformers are capable of solving these tasks but subquadratic alternatives to transformers cannot.

Document similarity is a fundamental area in natural language processing with many applications including recommender systems \citep{Ost20}, search engines \citep{MAI18}, and plagiarism-detection \citep{BNM17}. For a given document $D$, we first need to define a document embedding to transform it into a vector $v \in \mathbb{R}^d$, and we will then measure how similar two documents are by using a similarity measure on their embedding vectors. There are many ways to embed a document as a vector including Doc2Vec \citep{LM14}, TF-IDF \citep{SJ88}, BERT \citep{DCLT19}, bag-of-words \citep{Harris54}, and many ways to measure how similar two documents are including cosine similarity, Euclidean distance, and Jaccard Similarity. In this work, we focus on two of the most popular options, bag-of-words embedding and cosine similarity, although our results extend naturally to almost any reasonable alternatives. 

\paragraph{Bag-of-words Embedding.} Bag-of-words embedding is a well-studied method of embedding \citep{BNJ03} that is commonly used for text classification \citep{JZCX16}, radiology \citep{JSKE21} and many other settings, especially in NLP. Given a document $D$ and a list of $\ell$ key words, the bag-of-words embedding of $D$ is a vector $v \in \{0,1\}^\ell$ such that the $i$-th entry of $v$ corresponds to whether the $i$-th key word exists in $D$ or not. \footnote{Sometimes the entries are also frequency counts, but that would only make the embedding vectors more complex. Since our main goal is to show hardness, we will focus on the simpler binary case.} For example, if a document only contains one sentence ``\textit{There are ten apples on the apple tree}" and the key words chosen are ``\textit{apple}", ``\textit{tree}", ``\textit{computer}" and ``\textit{ten}", this document will have bag-of-words embedding $(1,1,0,1)$.

\paragraph{Cosine Similarity.} Cosine similarity is one of the most commonly used method to measure how similar two documents are. Given two document embeddings $v,w \in \mathbb{R}^d$, the cosine similarity is given by 
\[
\frac{\langle v,w\rangle}{\|v\|_2\cdot \|w\|_2} \in [0,1]
\] where $1$ represents complete similarity and $0$ represents no similarity. It is defined in this way, rather than just taking the inner product $\langle v,w\rangle$, so that two documents (vectors) with large magnitude can still be considered close if they have similar directions. Cosine similarity is one of the most popular and effective measures; for instance, \cite{SVP21} found that when used for extractive multi-document text summarization, cosine similarity gives the best results.

\paragraph{Problem Statement.} We define our most similar document $(\MSD)$ and least similar document $(\LSD)$ problems as: given a set of $n$ binary vectors $v_1,\ldots,v_n$ of length $d$ (document embeddings), the goal is to find two documents that are the most/least similar to each other in terms of cosine similarity. There are many natural variants of these two problems, and we prove similar hardness for all of them:
\begin{enumerate}
    \item (Bichromatic $\MSD,\LSD$) Sometimes we have two sets of documents $A,B$ and we want to find one document from each set such that the pair is (un)similar.
    \item ($\gamma\textup{-}\MSD, \gamma\textup{-}\LSD$) Sometimes we might only need to find a pair of documents that is \emph{approximately} the most (un)similar, (up to an approximation factor $\gamma$) and not necessarily the optimal pair.
    \item $(\MSD_{n,d,t},\LSD_{n,d,t})$ Sometimes we only want to know if there exists a pair whose cosine similarity is above (or below) a threshold $t \in [0,1]$.
\end{enumerate}

These variants occur in many practical scenarios when using a large language model. They can arise explicitly when the descriptions of $n$ documents of size $\ell$ are given to a language model, and the model is asked to find the most similar pair of documents. However, there are many scenarios where document similarity can arise implicitly as well, such as in plagiarism detection and team matching.

\subsection{Main Results}
\label{sec: main results}

Our hardness results are based on a prevalent conjecture from fine-grained complexity theory called the Strong Exponential Time Hypothesis $(\SETH)$:
\begin{center}
    \emph{For every $\varepsilon>0$, there is an integer $k$ such that $k\SAT$ with $n$ variables requires $\Omega(2^{(1-\varepsilon)n})$ time}.
\end{center} $\SETH$ was first introduced by \cite{IP01}, and is a popular strengthening of the conjecture that $\PP \neq \NP$. (In other words, proving that $\SETH$ is true implies $\PP \neq \NP$.) Since then, there has been a long line of work studying and making use of $\SETH$. Prior work has given theoretical evidence for $\SETH$~\citep{IP01,ABDN18,Williams15talk}, and has used $\SETH$ to prove hardness of problems in many different areas of algorithm design. See the survey \cite{Williams18} for a detailed background.

\paragraph{Main Results: Limitations of subquadratic alternatives.} We show that $\MSD,\LSD$ and their variants require quadratic time assuming $\SETH$ for some natural choice of parameters, and therefore any subquadratic alternatives to transformers are not able to solve them due to computational constraints. The formal hardness results are as follows, and vary slightly in the dimension parameter $\ell$ depending on the details of the problem:
\begin{theorem}[\Cref{thm: approximate LSD is hard} and Corollary \ref{cor: variants of LSD are OV hard}]
Assuming $\SETH$, for every $\varepsilon>0$, there exists a constant $c>0$ such that $\LSD_{n,\ell}$ cannot be solved in $O(n^{2-\varepsilon})$ time when $\ell = c\log n$. Moreover, the same lower bound also holds for $\LSD_{n,\ell,t}$ for some $0< t < 1$,  $\gamma\text{-}\LSD_{n,\ell}$ for any $\gamma \geq 1$, and bichromatic $\LSD_{n,\ell}$.
\end{theorem}

\begin{theorem}[\Cref{thm: approximate MSD is hard} and Corollary \ref{cor: variants of MSD are OV hard}]
Assuming $\SETH$, for every $\varepsilon>0$, there exists a constant $c>0$ such that $\MSD_{n,\ell}$ cannot be solved in $O(n^{2-\varepsilon})$ time when $\ell = n^{\frac{c}{\log\log n}}$. Moreover, the same lower bound also holds for $\MSD_{n,\ell,t}$ for some $0<t<1$ and $\gamma\text{-}\MSD_{n,\ell}$ for any $1 \leq \gamma \leq \polylog(n)$.
\end{theorem}

For bichromatic $\MSD$ we can obtain a stronger hardness result.

\begin{theorem}[\Cref{thm: approximate MSD is hard} and Corollary \ref{cor: variants of MSD are OV hard}]
Assuming $\SETH$, for every $\varepsilon>0$, there exists a constant $c>0$ such that bichromatic $\MSD_{n,\ell}$ cannot be solved in $O(n^{2-\varepsilon})$ time when $\ell = c\log n$.
\end{theorem}

In all these problems, we prove hardness when $\ell$ is $\Theta(\log n)$ or $n^{o(1)}$. This is a natural choice: for any smaller $\ell<\log n$, there could be at most $2^{\ell}<n$ vectors under a bag-of-words embedding, which means that there would be duplicate vectors in our instance. Making $\ell$ larger can only make the problem harder.

It follows that any language task which, either explicitly or implicitly, involves solving any of these document similarity problems, cannot be solved in subquadratic time \emph{when the input size is large enough}, no matter what the parameters or architecture of the language model are.\footnote{We briefly emphasize that the parameter $\ell$ in these similarity problems need not be related to the architecture or parameters (like $d_{\inn}, d_{\out}$, etc) of a language model which solves the problems. For instance, to ask a language model to solve $\MSD_{n,\ell}$, we may ask it ``Which of the following paragraphs is most similar?'' followed by a list of $n$ different paragraphs of at most $\ell$ words each. Thus, the input to the language model would be a string of length $O(n \ell)$, and since $\ell < n^{o(1)}$ is small compared to $n$, our result shows that a language model would need to take quadratic time in the string length to answer this type of query.}

\paragraph{Main Results: Representational strength of standard transformers.}

When a problem requires quadratic time to solve, this means that subquadratic-time models cannot solve it, but it does not necessarily mean that a transformer model can solve it. For example, \cite{SHT24B} defined a problem called ``Match3'' which can be solved in quadratic time by a textbook algorithm, but which they prove cannot be solved by a one-layer transformer unless it has a lot of attention heads or a very high embedding dimension.

We show that this is not the case for $\MSD$ and $\LSD$ by showing that a single standard attention unit with input and output MLPs can solve $\MSD_{n,d,t},\LSD_{n,d,t}$ and a simpler version of $\MSD$, the Orthogonal Vectors problem ($\OV$), where one is given a set of binary vectors and needs to determine if there exists a pair of vectors that are orthogonal. Thus, these problems establish a separation between standard transformers and subquadratic alternatives to transformers.

\begin{theorem}[\Cref{thm: transformer solves OV,thm: transformers can solve SD}]
A single unit of standard attention with input and output MLPs, embedding dimension $\ell+1$ can solve $\OV_{n,\ell}$ and $\MSD_{n,\ell,t},\LSD_{n,\ell,t}$ for any $0 \leq t \leq 1$.
\end{theorem}

In principle, there could be concerns with representational results like these that the weights of the model are complicated and hard to find in training. However, our constructions of transformers that solve these problems are very simple: our MLPs are piece-wise linear functions that are easy to approximate/compute, and our key/query/value matrices in the attention unit are also simple, sparse and low-rank matrices with small entries.

\subsection{Related Work}

In recent years, several theoretical and algorithmic aspects of transformers have been extensively studied. We discuss next two aspects that are most relevant to our work, and which we build on in the proofs of our results.

\paragraph{Representational strengths of Transformers.} Representational strengths of transformers have been widely studied in recent years. It has been shown that transformers have several natural limitations, including not being able to model periodic
finite-state languages or hierarchical structure \citep{Hahn20}, and not being able to recognize some counter languages without large depth \citep{BAG20}. On the other hand, transformers are able to recognize formal languages such as Dyck languages \citep{YPPN21}, simulate finite-state automata \citep{LAGKZ23} with $O(\log n)$ depth, and simulate Turing machines if given enough depth \citep{WCM24,MS24}. There is also a line of work \citep{HAF22,MSS22} that understands what transformers can compute through the lens of circuit complexity; see \cite{SMWCA24} for a comprehensive survey.

Another line of work has shown that transformers can compute particular problems of interest, including PARITY (which perceptrons are notably unable to compute) \citep{CC22} and learning problems that attention is particularly suited toward like ``sparse averaging'' and ``$k$-hop induction heads'' \citep{SHT24,SHT24B}.

\paragraph{Faster attention mechanisms.} As previously discussed, attention computation remains a bottleneck for efficiency, and many different approaches have been proposed to tackle this issue. These approaches typically result in accuracy loss (and a consequence of our main result is theoretical proof that this is necessary), which has mostly been investigated empirically.

The main prior work on theoretical limitations of subquadratic transformers we're aware of is \cite{SHT24}. Among other results, they study ``kernel-based subquadratic attention'' in which one heuristically computes attention faster by approximating intermediate matrices in the attention computation either by sparse matrices \citep{KKL20, routingtransformer21, DKOD20} or low-rank matrices \citep{CLD+21,KVPF20}. \cite{SHT24} defined a ``$k$-hop induction heads'' task and proved that transformers can perform this task but kernel-based subquadratic attention models cannot. Our limitation result is more general than this, applying to any approach that runs in subquadratic time, rather than needing to focus on a particular architecture or heuristic.

\paragraph{Fine-Grained Complexity and Machine Learning}

Fine-grained complexity theory has been successful at proving conditional lower bounds for problems in diverse areas of algorithm design, such as in graph theory \citep{AW14,ABW15,WW18} and combinatorial optimization \citep{Rub18,ATYZ24,KPS17,BI15}. See \cite{Williams18} for a detailed survey.

Recently, it has been shown that many problems in machine learning are also inherently hard assuming popular conjectures in fine-grained complexity. For example, \cite{BIS17,AG24} use $\SETH$ to give a lower bound on the time to perform kernel density estimation, \cite{HLSL} use $\SETH$ to demonstrate a computational phase transition in modern Hopfield models, \cite{HWLPSL24} and \cite{HSKSL24} use $\SETH$ to characterize the computational limits of diffusion transformers and Low-Rank Adaptation for transformers respectively, and \cite{KWH23,AS24} use $\SETH$ to give a lower bound on computational complexity of approximating the attention mechanism itself. We remark that our results, together, give a new, alternate proof of the hardness of attention assuming $\SETH$. Indeed, we prove that a single attention unit can solve $\MSD$, and that $\MSD$ requires quadratic time assuming $\SETH$, which together imply that evaluating the attention unit requires quadratic time.

\paragraph{Hardness of similarity search.} Similarity search has been a fundamental area in modern machine learning, and the efficiency of similarity search algorithms has been well-studied through the lens of fine-grained complexity. Perhaps the most well-known problem in this area is the nearest neighbor problem.

Nearest Neighbor is a fundamental problem in machine learning which has been the subject of decades of research \citep{IM98,AI08,AR15,ALRW17,TPS02,ECS24,UHLMG22,GAR19}. Given a dataset $P \subseteq \mathbb{R}^d$ with $n$ points, we want to preprocess it such that given a query point $q \in \mathbb{R}^d$, one can find its nearest neighbor in $P$ (in some metric) efficiently. The nearest neighbor problem gives one common way of classifying objects in machine learning: given an object with unknown label, one just find its nearest neighbor in the dataset and use its label as the label for the target object. In addition, it has many applications in classical similarity search over different types of data including text, images, audio (see \cite{SDI06} for a complete overview).

There are many natural variants of the nearest neighbor problem, including the closest (furthest) pair problem where one is given a dataset $P$ and wants to find the two points from $P$ that are the closest (furthest). This problem is exactly $\MinIP(\MaxIP)$ if we let inner product be the measure of closeness. Under standard Euclidean distance, one can also express the distance between two points as $\|p-q\| = \|p\|+\|q\|-2\langle p,q\rangle$ such that when all points have the same $\ell_2$ norm, finding the closest pair is equivalent to finding the pair with the largest inner product, i.e. $\MaxIP$. In fact, it has been shown that assuming $\SETH$ is true, finding closest pairs in Euclidean or Manhattan distance both require quadratic time \citep{AW15, Rub18}. We will use this in our proofs below. 


\section{Preliminaries}

\paragraph{Notation.}  For a vector $v$, we use $v[i]$ to denote its $i$-th entry for all $i$. For a matrix $A$, we use $A_{i,:}$ to denote the $i$-th row of $A$ and $A_{:,j}$ to denote the $j$-th column of $A$. Given a positive integer $d$, we use $\mathbf{1}_{d} \in \mathbb{R}^d$ to denote the vector whose entries are all $1$. Given two vectors $v \in \mathbb{R}^a, w \in \mathbb{R}^b$, we use $v \otimes w \in \mathbb{R}^{ab}$ to denote the Kronecker product of $v$ and $w$ (whose entries are all the products of an entry of $v$ and an entry of $w$) and $v\circ w \in \mathbb{R}^{a+b}$ to denote the concatenation of $v$ and $w$. $\|\cdot\|$ refers to $\ell_2$ norm unless otherwise specified. Given a binary vector $v \in \{0,1\}^d$, we use $\Bar{v} \in \{0,1\}^d$ to denote the vector where all entries are flipped. The $\softmax$ operator, when given a vector $v \in \mathbb{R}^n$, outputs a vector in $\R^n$ given by
\[
\softmax(v) = \frac{(\exp(v[1]),\ldots,\exp(v[n]))}{\sum_{i=1}^{n}\exp(v[i])}.
\] For matrices $A \in \mathbb{R}^{n \times n}$, we apply $\softmax$ operator row-wise, so that 
\[
\softmax(A)_{i,:} = \softmax(A_{i,:}).
\]

\subsection{The Transformer Architecture}

Transformer is a machine learning architecture composed mainly of attention layers and multi-layer perceptrons (MLP). We model the input to a attention unit is a $n \times d_{\inn}$ matrix where $d_{\inn}$ is the input dimension and the output of a attention is a $n \times d_{\out}$ matrix where $d_{\out}$ is the output dimension.

\begin{definition}[attention]
For input dimension $d_{\inn} \in \mathbb{N}$, output dimension $d_{\out} \in \mathbb{N}$, embedding dimension $m \in \mathbb{N}$, matrices $Q,K \in \mathbb{R}^{d_{\inn} \times m}$ and $V \in \mathbb{R}^{d_{\inn} \times d_{\out}}$, a \emph{attention} is a mapping $A_{Q,K,V}: \mathbb{R}^{n \times d_{\inn}} \rightarrow \mathbb{R}^{n \times d_{\out}}$ by 
\[
A_{Q,K,V}(X) = \softmax(XQK^{\top}X^{\top})XV.
\] We use $\mathcal{A}_{d_{\inn},m,d_{\out}} = \{A_{Q,K,V}: Q \in \mathbb{R}^{d_{\inn} \times m},K \in \mathbb{R}^{d_{\inn}\times m},V \in \mathbb{R}^{d_{\inn} \times d_{\out}}\}$ to denote all such attentions.
\end{definition}

An attention layer consists of many attentions in parallel. Upon receiving input $X$, each attention computes an output locally, then the results are all concatenated into a large matrix before being sent to the next layer. Our constructions in this paper will only need transformers with one layer and one single unit of attention to illustrate representational strength (transformers with more layers and heads could only be stronger), so we omit the formal definitions of attention layers.

A \emph{multi-layer perceptron} (MLP) is a type of neural network that is used to learn nonlinear relationships in data. Mathematically, it is usually formulated as a neural network with different types of activation functions \citep{BFT17,MPCB14,JGH18} or sometimes as a more specific threshold circuit \citep{MSS94}. Since the universal approximation theorem \citep{HSW89} states that any continuous function with a finite support can be approximated by a neural network with one hidden layer, \cite{SHT24,SHT24B} modeled MLP as an arbitrary function $\phi:\mathbb{R}^d \rightarrow \mathbb{R}^{d'}$ defined on fixed-precision vectors, and we will notationally use that definition here.

\begin{definition}[Multi-player perceptron]
A multi-layer perceptron is represented by some continuous function $\varphi:\mathbb{R}^{a} \rightarrow \mathbb{R}^b$ for some positive integers $a,b$. We can apply $\varphi$ to a matrix row-wisely: given any matrix $X \in \mathbb{R}^{n \times a}$, $\varphi(X) = (\varphi(X_1),\ldots,\varphi(X_n)) \in \mathbb{R}^{n \times b}$.
\end{definition}

That said, in all our constructions in this paper, it will suffice to use MLPs which are simple, piecewise-linear functions that fit in all the categories discussed above, and can be easily approximated by a small neural network or any other MLP definition.

In this work, our transformer (with a single attention unit) will be defined as a composition of the first MLP, then one attention unit, then the second MLP. This is a natural model for many well-known transformer models including BERT \citep{BERT}, GPT-3 \cite{GPT3}, GPT-4 \citep{GPT4} and is typically used in theoretical work on simple transformers~\citep{PH22,SHT24,SHT24B}.

Recall in particular that we will be designing such transformers for document similarity problems. In this case, given $n$ documents $D_1,\ldots,D_n$, we will let their bag-of-words embeddings 
\[
(\BOW(D_1),\ldots,\BOW(D_n))^{\top} \in \mathbb{R}^{n \times d}
\] be the input to our transformer, and the output of the transformer will be a real number indicating the answer to our problems. 

\begin{definition}
A \emph{transformer} is a mapping $\TF: \mathbb{R}^{n \times d} \rightarrow \mathbb{R}$ specified by a attention unit $A_{Q,K,V}$ and two MLPs $\varphi_1: \mathbb{R}^{n\times d} \rightarrow \mathbb{R}^{n\times d_{\inn}},\varphi_2: \mathbb{R}^{n \times d_{\out}} \rightarrow \mathbb{R}$. Upon an embedding matrix $E \in \mathbb{R}^{n \times d}$, the transformer  outputs $\varphi_2(A_{Q,K,V}(\varphi_1(E)))$.
\end{definition}

To emphasize, this is a very simplified model of a transformer with a single attention unit. We say that a transformer $\TF$ solves a problem whose input is a matrix $E \in \mathbb{R}^{n \times d}$ if $\TF(E)$ is the answer of the problem on instance $E$. For example, for decision problems like $\MSD_{n,d,t}$, we say $\TF$ solves $\MSD_{n,d,t}$ if for all input $v_1,\ldots,v_n$ and $E$ such that $E_{i,:} = v_i$ for all $i$, $\TF(E) = 1$ if there exists a pair $\frac{\langle v_i,v_j\rangle}{\|v_i\|\cdot \|v_j\|} \geq t$ and $\TF(E) = 0$ otherwise.

\subsection{Fine-Grained Complexity}
\label{sec: fine-grained complexity}

We first introduce some common notions from fine-grained complexity.
Many proofs in this section are deferred to  \Cref{sec: detailed FGC}, and we also refer the reader to \Cref{sec: detailed FGC} for a more detailed introduction to fine-grained complexity.

In fine-grained complexity, one is usually interested in whether we can improve the running time of our algorithms by a polynomial factor. For example, the $\OV$ problem (defined below) has a straightforward quadratic time (ignoring logarithmic factors) algorithm, and it is a major open problem to determine whether there exists a faster, $O(n^{1.99})$ time algorithm. We say that an algorithm is \emph{truly subquadratic} \footnote{Usually in fine-grained complexity, a subquadratic time algorithm means the algorithm runs in time $o(n^2)$, and a truly subquadratic time algorithm means the algorithm runs in time $O(n^{2-\varepsilon})$ for a fixed constant $\varepsilon>0$. For example, $O(n^{2} / \log n)$ is subquadratic but not truly subquadratic. By contrast, in the context of fast attention, prior machine learning literature often just calls an approach subquadratic to mean that it is truly subquadratic, or often even almost linear time. In this work, we will use the fine-grained complexity definition and refer to such approaches as truly subquadratic.} if it runs in time $O(n^{2-\varepsilon})$ for some constant $\varepsilon>0$. In this work, all our problems have quadratic time solutions, and we are interested in whether truly subquadratic time algorithms exist.

A key technique (that we will use frequently in this work) in fine-grained complexity is the \emph{fine-grained reduction}, which is a way to connect the running times of different problems. If problems $\mathcal{P}$ and $\mathcal{Q}$ both have quadratic time algorithms, we say that $\mathcal{P}$ \emph{reduces to} $\mathcal{Q}$ (sometimes we would also say $\mathcal{P}$ is easier than $\mathcal{Q}$ or $\mathcal{Q}$ is harder than $\mathcal{P}$) if a truly subquadratic time algorithm for $\mathcal{Q}$ implies a truly subquadratic time algorithm for $\mathcal{P}$. We say $\mathcal{P}$ and $\mathcal{Q}$ are \emph{subquadratic equivalent} if they reduce to each other, i.e. there is a truly subquadratic time algorithm for $\mathcal{P}$ if and only if there is a truly subquadratic time algorithm for $\mathcal{Q}$. Such relationships are proved by careful reductions, and we will see many examples soon.

\subsubsection{Hardness Conjectures: $\SETH$ and $\OVC$}

We restate our central hardness conjecture here.

\begin{definition}[Strong Exponential Time Hypothesis ($\SETH$)]
For any $\varepsilon>0$, there exists a positive integer $k$ such that $k\SAT$ requires $\Omega(2^{(1-\varepsilon)n})$ time, where $n$ is the number of variables in the CNF.
\end{definition}

$\SETH$ has been one of the biggest open problem in fine-grained complexity, and one of the reasons is that $\SETH$ being true would imply $\PP \neq \NP$. See \Cref{sec: main results} for a detailed explanation of its significance and why many people believe that it is true.

We also introduce the Orthogonal Vectors problem, which is an important problem in fine-grained complexity and will be a key intermediate problem in some of our proofs.

\begin{definition}[Orthogonal Vectors ($\OV_{n,\ell}$)]
Given binary vectors $v_1,\ldots,v_n \in \{0,1\}^\ell$, $\OV_{n,\ell}$ asks to determine if there exists a pair $i \neq j$ such that $\langle v_i,v_j\rangle = 0$.
\end{definition}

Much effort has been made to give a truly subquadratic algorithm for $\OV_{n,c\log n}$ for all $c$, but none has succeeded. Therefore, \cite{Williams05} proposed the Orthogonal Vectors Conjecture, which asserts that such an algorithm does not exist.

\begin{conjecture}[$\OVC$]
For any $\varepsilon>0$, there exists a constant $c>0$ such that $\OV_{n,c\log n}$ cannot be solved in $O(n^{2-\varepsilon})$ time.
\end{conjecture}

\cite{Williams05} showed that assuming $\SETH$ is true, then $\OVC$ is true (the other direction is unknown). Our paper will use these two conjectures interchangeably such that our hardness results can be obtained from either conjecture.

There is also a bichromatic version of $\OV_{n,\ell}$ where one is given two sets of vectors $A = \{a_1,\ldots,a_n\}, B = \{b_1,\ldots,b_n\}$ such that $a_i,b_j \in \{0,1\}^\ell$ and wants to determine if there exists $i,j$ such that $\langle a_i,b_j\rangle = 0$. In fact, these two problems are subquadratic equivalent (see \Cref{lem: bichromatic OV equivalent to OV} for proof). 

\subsubsection{Minimum Inner Product}

In this section we introduce the minimum inner product problem, an important problem related to similarity search.

\begin{definition}[$\MinIP$]
Given a set of binary vectors $v_1,\ldots,v_n \in \{0,1\}^\ell$, $\MinIP_{n,\ell}$ asks to find one pair of $1 \leq i,j \leq n, i \neq j$ such that $\langle v_i,v_j\rangle$ is minimum.  
\end{definition} 

Sometimes we are happy with finding a pair of vectors whose inner product is close enough to optimal, so we also introduce the approximate $\MinIP$ problem as follows.

\begin{definition}[$\gamma\text{-}\MinIP$]
Given a set of binary vectors $v_1,\ldots,v_n \in \{0,1\}^\ell$, $\gamma\text{-}\MinIP_{n,\ell}$ asks to find one pair of $1 \leq i,j \leq n, i \neq j$ such that $\langle v_i,v_j\rangle$ is a $\gamma$-approximation of the minimal inner product.
\end{definition}

It is not hard to see that $\MinIP$ and $\gamma\text{-}\MinIP$ are both at least as hard as $\OV$ for any $\gamma \geq 1$ (just find the minimum inner product and see if it is $0$, and any multiplicative approximation of $0$ must be $0$). Therefore, assuming $\OVC$, for any $\varepsilon>0$ there exists $c>0$ such that $\MinIP_{n,c\log n}$ cannot be solved in $O(n^{2-\varepsilon})$ time. 

In addition, there is a decision version of $\MinIP$, by which we denote $\MinIP_{n,\ell,t}$, where one wants to know whether there exists a pair of vectors whose inner product is at most $t$ for some $0 \leq t\leq \ell$.

\begin{definition}[$\MinIP$ decision version]
Given a set of binary vectors $v_1,\ldots,v_n \in \{0,1\}^\ell$ and $0 \leq t \leq \ell$, $\MinIP_{n,\ell,t}$ asks to determine if there exists one pair of $1 \leq i,j \leq n, i \neq j$ such that $\langle v_i,v_j\rangle \leq t$. 
\end{definition}

The bichromatic versions of these problems can be defined analogously: given two sets $A,B$ with vectors in $\{0,1\}^\ell$, one needs to find $i,j$ that achieves (for bichromatic $\MinIP_{n,\ell}$) or approximates (for bichromatic $\gamma\text{-}\MinIP_{n,\ell}$) the minimal $\langle a_{i},b_{j}\rangle$. One can obtain a truly subquadratic algorithm for all three problems above given a truly subquadratic algorithm for their bichromatic versions (see \Cref{lem: MinIP reduces to bichromatic MinIP} for proof).

\subsubsection{Maximum Inner Product}

One can analogously define $\MaxIP$ and its variants; see \Cref{sec:appendeixmaxip} for the formal definitions.

It is less obvious whether $\MaxIP$ is a harder problem than $\OV$ or not. The answer is positive, see \Cref{lem: OV reduces to bichromatic MaxIP} for a simple proof.

In fact, \cite{KM20} proved a stronger statement which says that even approximate $\MaxIP_{n,\ell}$ is harder than $\OV_{n,\ell}$ for some approximation factor.

\subsection{Document Similarity Problems}
\label{sec: similar documents}

In this section we formally define the $\MSD,\LSD$ problems that we will study. First we define the $\MSD$ variants, which are defined similarly to $\MaxIP$ variants.

\begin{definition}[$\MSD$]
Given $n$ document embeddings $v_1,\ldots,v_n \in \{0,1\}^\ell$, $\MSD_{n,\ell}$ asks to find $1 \leq i,j \leq n, i \neq j$ such that $\frac{\langle v_{i},v_{j}\rangle}{\|v_i\|\cdot \|v_j\|}$ is the maximum. \footnote{In all versions of $\MSD$ and $\LSD$, we assume that there are no zero vectors.}
\end{definition}

Even though $\MSD$ looks similar to $\MaxIP$, notice that they are not the same problem because of normalization. For example, $v = (1,1,\ldots,1) \in \mathbb{R}^{\ell}$ and $w = (1,1,\ldots,1,0,\ldots,0) \in \mathbb{R}^\ell$ where $w$ has $\ell/2$ ones have a very large inner product but they might not be considered similar. In contrast, $v' = (1,1,\ldots,1,0,\ldots,0) \in \mathbb{R}^\ell$ where $v'$ has $10$ ones and $w' = (0,1,\ldots,1,0,\ldots,0) \in \mathbb{R}^\ell$ where $w'$ has $10$ ones have a inner product of $9$ but they are very similar in terms of cosine similarity.

\begin{definition}[$\gamma\text{-}\MSD$]
Given $n$ document embeddings $v_1,\ldots,v_n\in \{0,1\}^\ell$, $\gamma\text{-}\MSD_{n,\ell}$ asks to find $1 \leq i^{*},j^{*} \leq n, i^{*} \neq j^{*}$ such that
\[
\frac{1}{\gamma}\cdot \max_{1 \leq i,j \leq n}\frac{\langle v_{i},v_{j}\rangle}{\|v_i\|\cdot \|v_j\|} \leq \frac{\langle v_{i^{*}},v_{j^{*}}\rangle}{\|v_{i^{*}}\|\cdot \|v_{j^{*}}\|} \leq \max_{1 \leq i,j \leq n}\frac{\langle v_{i},v_{j}\rangle}{\|v_i\|\cdot \|v_j\|}.
\]
\end{definition}

\begin{definition}[$\MSD$ decision version]
Given $n$ document embeddings $v_1,\ldots,v_n \in \{0,1\}^\ell$, $\MSD_{n,\ell,t}$ asks to determine if there exists $1 \leq i,j \leq n, i \neq j$ such that $\frac{\langle v_{i},v_{j}\rangle}{\|v_i\|\cdot \|v_j\|} \geq t$.
\end{definition}

It is not hard to see that $\gamma\text{-}\MSD_{n,\ell}$ and $\MSD_{n,\ell,t}$ are both easier than $\MSD$. In addition, notice that the number of possible $t$ can be considered as discrete because there could only be $O(\ell^3)$ possible values of $\frac{\langle v_i,v_j\rangle}{\|v_i\|\cdot \|v_j\|}$. As a result, the existence of truly subquadratic time algorithm for $\MSD_{n,\ell,t}$ for all $t \in [0,1]$ would imply a truly subquadratic time algorithm for $\MSD_{n,\ell}$ using binary search.

Bichromatic versions of these problems can be defined analogously and the proof of Lemma \ref{lem: MinIP reduces to bichromatic MinIP} again tells us that bichromatic versions are harder. 

One can analogously define $\LSD$ and its variants; see \Cref{sec:appendixLSD} for the formal definitions.

\section{Hardness of Document Similarity}

In this section, we show that assuming $\SETH$ or $\OVC$, for any $\varepsilon>0$, there exists a constant $c>0$ (only depends on $\varepsilon$) such that many variants of $\LSD_{n,c\log n},\MSD_{n,c\log n}$ require $O(n^{2-\varepsilon})$ time. 

\begin{theorem}
\label{thm: approximate LSD is hard}
Assuming $\SETH$ or $\OVC$, for every $\varepsilon>0$, there exists a constant $c>0$ such that $\gamma\text{-}\LSD_{n,\ell}$ cannot be solved in $O(n^{2-\varepsilon})$ time for any $\gamma \geq 1$ when $\ell = c\log n$.
\end{theorem}
\begin{proof}
Assume by contradiction that there exists an algorithm $\mathcal{A}$ for $\gamma\text{-}\LSD_{n,c\log n}$ that runs in time $O(n^{2-\varepsilon})$ for some $\gamma \geq 1, \varepsilon>0$ and any constant $c>0$. We show that $\OV_{n,c\log n}$ can be solved in time $O(n^{2-\varepsilon})$ for any constant $c>0$, which refutes $\OVC$ and $\SETH$. 

Given vectors $v_1,\ldots,v_n \in \{0,1\}^\ell$ where $\ell = c\log n$ for any constant $c$, if any vector is the zero vector (we can check in time $O(n\ell)$), then output yes. Otherwise we run $\mathcal{A}$ on $v_1,\ldots,v_n$ to compute $i^{*},j^{*}$ such that  
\[
\min_{1 \leq i,j \leq n}\frac{\langle v_i,v_j\rangle}{\|v_i\|\cdot \|v_{j}\|} \leq \frac{\langle v_{i^{*}},v_{j^{*}}\rangle}{\|v_{i^{*}}\|\cdot \|v_{j^{*}}\|} \leq \gamma\cdot \min_{1 \leq i,j \leq n}\frac{\langle v_i,v_j\rangle}{\|v_i\|\cdot \|v_{j}\|}.
\] Observe that $\displaystyle\min_{1 \leq i,j \leq n}\frac{\langle v_i,v_j\rangle}{\|v_i\|\cdot \|v_{j}\|} = 0$ if and only if there exists a pair of orthogonal vectors, which implies that there exists a pair of orthogonal vectors if and only if $\mathcal{A}$ outputs a pair of orthogonal vectors. The total amount of time needed for $\OV_{n,\ell}$ is therefore $O(n\ell+n^{2-\varepsilon}) =O(n^{2-\varepsilon})$, which refutes $\SETH$.
\end{proof}

Since $\gamma\text{-}\LSD_{n,\ell}$ is easier than $\LSD_{n,\ell}$ and bichromatic $\gamma\text{-}\LSD_{n,\ell}$, the same lower bound applies to these two problems as well. In addition, there must exist $t \in [0,1]$ such that $\LSD_{n,\ell,t}$ cannot be solved in $O(n^{2-\varepsilon})$ time when $\ell = c\log n$ because otherwise that would imply a $O(n^{2-\varepsilon})$ time algorithm for $\LSD_{n,\ell}$ using binary search.

\begin{corollary}
\label{cor: variants of LSD are OV hard}
Assuming $\SETH$ or $\OVC$, for every $\varepsilon>0$, there exists a constant $c>0$ such that $\LSD_{n,d}$ cannot be solved in $O(n^{2-\varepsilon})$ time when $\ell \geq c\log n$. Moreover, the same lower bound holds for bichromatic $\gamma\textup{-}\LSD_{n,\ell}$ for all $\gamma \geq 1$ and $\LSD_{n,\ell,t}$ for some $t \in [0,1]$.
\end{corollary}

A similar hardness result for $\gamma\textup{-}\MSD_{n,d}$ can be derived with much more complicated techniques. Our proof follows the same idea as Theorem 6.1 of \cite{KM20} which uses graph constructions, and we delay the proof of \Cref{thm: approximate MSD is hard} to \Cref{sec: normalized MaxIP is hard}.

\begin{theorem}
\label{thm: approximate MSD is hard}
Assuming $\SETH$, for every $\varepsilon>0$, there exists a constant $c>0$ such that $\gamma\textup{-}\MSD_{n,\ell}$ cannot be solved in $O(n^{2-\varepsilon})$ time when 
\[
\ell \geq (\log n)^{\frac{c\log n}{(\log\log n)^2}} \textup{ and } \gamma \leq \Big(1+\frac{1}{\log\log n}\Big)^{\frac{\log n}{(\log\log n)^2}}.
\]
\end{theorem}

Similarly, the hardness result also applies to harder problems including $\MSD_{n,\ell}$ and bichromatic $\gamma\textup{-}\MSD_{n,\ell}$.

\begin{corollary}
\label{cor: variants of MSD are OV hard}
Assuming $\SETH$ or $\OVC$, for every $\varepsilon>0$, there exists a constant $c>0$ such that $\MSD_{n,\ell}$ cannot be solved in $O(n^{2-\varepsilon})$ time when $\ell \geq (\log n)^{\frac{c\log n}{(\log\log n)^2}}$. Moreover, the same lower bound holds for bichromatic $\gamma\textup{-}\MSD_{n,\ell}$ for all $1 \leq \gamma \leq (1+\frac{1}{\log\log n})^{\frac{\log n}{(\log\log n)^2}}$ and $\MSD_{n,\ell,t}$ for some $t \in [0,1]$.
\end{corollary}

\section{Representational Strength of Transformers}

So far we have seen multiple problems ($\OV,\MinIP,\MaxIP$ and variants of $\MSD,\LSD$) that require quadratic time to solve under certain parameters assuming $\SETH$ or $\OVC$. In this section, we show that $\OV$ and decision versions of $\MSD,\LSD$ can be solved by a transformer with one attention unit in one layer.

Notice that fine-grained reduction does not trivially apply in representational strength of transformers, i.e. two problems might be subquadratic equivalent, but one problem solvable by transformers might not imply that the other one is also solvable by transformers. This is because many techniques that are simple to do on the word-RAM model (where one can do arithmetic operations, find the maximum/minimum over $n$ numbers in constant time) might not be easy to implement in parallel architectures like transformers.

We show that transformers are able to solve $\OV$ with appropriate parameters. The constructions for $\MinIP,\MaxIP,\MSD,\LSD$ are more complicated but follow similar ideas, so we leave them to \Cref{sec: representational stength proofs}.

\begin{theorem}
\label{thm: transformer solves OV}
An attention unit with input and output MLPs with parameters $d = \ell, d_{\inn} = \ell, d_{\out} = 1, m \geq \ell+1$ can solve $\OV_{n,\ell}$.
\end{theorem}
\begin{proof}
Let $v_1,\ldots,v_n \in \{0,1\}^\ell$ be an $\OV_{n,\ell}$ instance; define $v_{n+1} := 0^\ell$ and $X \in \mathbb{R}^{(n+1) \times \ell}$ such that $X_{i,:} = v_i$ for all $i$. Since $m \geq \ell+1$, let $Q,K$ be arbitrary matrices such that $QK^{\top} = -3\log n\cdot I_{\ell}$ and $V \in \mathbb{R}^{\ell \times 1}$ be the all-one matrix. Let $A_{Q,K,V}$ denote this attention head, and $\varphi_1$ be the identity function. 

Let $X$ be the input to the transformer with $A_{Q,K,V}$ as the only attention head. We claim that if there exists a pair $1 \leq i \neq j \leq n$ such that $\langle v_i,v_j\rangle = 0$, then $A_{Q,K,V}(X)$ has an entry which is at least $\frac{1}{n+1}$, and otherwise all entries will be at most $\frac{1}{(n+1)^{1.5}}$. As a result, we can use the second MLP $\varphi_2$ to map $A_{Q,K,V}(X)$ to $1$ if any of its entry is at least $\frac{1}{n+1}$ and $0$ otherwise (see Lemma \ref{lem: second MLP} for a formal proof).
 
We can calculate that the $i$-th entry of $A_{Q,K,V}(X)$ is 
\[
\sum_{j=1}^{n}\frac{\exp(-3\log n\cdot \langle v_i,v_j\rangle)}{\sum_{k=1}^{n}\exp(-3\log n\cdot \langle v_i,v_{k}\rangle)+1}\cdot\|v_j\|_1 = \frac{\sum_{j=1}^{n}n^{-3\langle v_i,v_j\rangle}}{\sum_{k=1}^{n}n^{-3\langle v_i,v_k\rangle}+1}\cdot \|v_j\|_1.
\] As a result, if there exists a pair $\langle v_{i^{*}},v_{j^{*}}\rangle = 0$ with $1 \leq i^{*},j^{*} \leq n$, then the $i^{*}$-th entry of $A_{Q,K,V}(X)$ can be lower bounded as
\[
\frac{\sum_{j=1}^{n}n^{-3\langle v_i,v_j\rangle}}{\sum_{k=1}^{n}n^{-3\langle v_i,v_k\rangle}+1} \geq \frac{n^{-3\langle v_{i^{*}},v_{j^{*}}\rangle}}{n+1} = \frac{1}{n+1}.
\] Otherwise, $\langle v_i,v_j\rangle \geq 1$ for all $i,j$ and thus the $i$-th entry of $A_{Q,K,V}(X)$ is 
\[
\frac{\sum_{j=1}^{n}n^{-3\langle v_i,v_j\rangle}}{\sum_{k=1}^{n}n^{-3\langle v_i,v_k\rangle}+1}\cdot \|v_j\|_1 \leq \sum_{j=1}^{n}\frac{d}{n^{3\langle v_i,v_j\rangle}} \leq n\cdot \frac{d}{n^{3}} < \frac{1}{n^{1.5}}.
\] 
\end{proof}

\bibliography{iclr2025_conference}

\begin{thebibliography}{81}
\providecommand{\natexlab}[1]{#1}
\providecommand{\url}[1]{\texttt{#1}}
\expandafter\ifx\csname urlstyle\endcsname\relax
  \providecommand{\doi}[1]{doi: #1}\else
  \providecommand{\doi}{doi: \begingroup \urlstyle{rm}\Url}\fi

\bibitem[CST(2017)]{CST}
Pairwise comparison of bit vectors.
\newblock Computer Science Theory Stack Exchange, 2017.
\newblock URL \url{https://cstheory.stackexchange.com/questions/37361/pairwise-comparison-of-bit-vectors}.
\newblock [Online:] \url{https://cstheory.stackexchange.com/questions/37361/pairwise-comparison-of-bit-vectors}.

\bibitem[Abboud \& Williams(2014)Abboud and Williams]{AW14}
Amir Abboud and Virginia~Vassilevska Williams.
\newblock Popular conjectures imply strong lower bounds for dynamic problems.
\newblock In \emph{Proceedings of the 2014 IEEE 55th Annual Symposium on Foundations of Computer Science}, FOCS '14, pp.\  434–443, USA, 2014. IEEE Computer Society.
\newblock ISBN 9781479965175.
\newblock \doi{10.1109/FOCS.2014.53}.
\newblock URL \url{https://doi.org/10.1109/FOCS.2014.53}.

\bibitem[Abboud et~al.(2015{\natexlab{a}})Abboud, Backurs, and Williams]{ABW15}
Amir Abboud, Arturs Backurs, and Virginia~Vassilevska Williams.
\newblock Tight hardness results for lcs and other sequence similarity measures.
\newblock In \emph{Proceedings of the 2015 IEEE 56th Annual Symposium on Foundations of Computer Science (FOCS)}, FOCS '15, pp.\  59–78, USA, 2015{\natexlab{a}}. IEEE Computer Society.
\newblock ISBN 9781467381918.
\newblock \doi{10.1109/FOCS.2015.14}.
\newblock URL \url{https://doi.org/10.1109/FOCS.2015.14}.

\bibitem[Abboud et~al.(2015{\natexlab{b}})Abboud, Williams, and Yu]{ARY15}
Amir Abboud, Ryan Williams, and Huacheng Yu.
\newblock More applications of the polynomial method to algorithm design.
\newblock In \emph{Proceedings of the Twenty-Sixth Annual ACM-SIAM Symposium on Discrete Algorithms}, SODA '15, pp.\  218–230, USA, 2015{\natexlab{b}}. Society for Industrial and Applied Mathematics.

\bibitem[Abboud et~al.(2018)Abboud, Bringmann, Dell, and Nederlof]{ABDN18}
Amir Abboud, Karl Bringmann, Holger Dell, and Jesper Nederlof.
\newblock More consequences of falsifying seth and the orthogonal vectors conjecture.
\newblock In \emph{Proceedings of the 50th Annual ACM SIGACT Symposium on Theory of Computing}, STOC 2018, pp.\  253–266, New York, NY, USA, 2018. Association for Computing Machinery.
\newblock ISBN 9781450355599.
\newblock \doi{10.1145/3188745.3188938}.
\newblock URL \url{https://doi.org/10.1145/3188745.3188938}.

\bibitem[Alman \& Guan(2024)Alman and Guan]{AG24}
Josh Alman and Yunfeng Guan.
\newblock Finer-grained hardness of kernel density estimation.
\newblock In Rahul Santhanam (ed.), \emph{39th Computational Complexity Conference, {CCC} 2024, July 22-25, 2024, Ann Arbor, MI, {USA}}, volume 300 of \emph{LIPIcs}, pp.\  35:1--35:21. Schloss Dagstuhl - Leibniz-Zentrum f{\"{u}}r Informatik, 2024.
\newblock \doi{10.4230/LIPICS.CCC.2024.35}.
\newblock URL \url{https://doi.org/10.4230/LIPIcs.CCC.2024.35}.

\bibitem[Alman \& Song(2024)Alman and Song]{AS24}
Josh Alman and Zhao Song.
\newblock Fast attention requires bounded entries.
\newblock In \emph{Proceedings of the 37th International Conference on Neural Information Processing Systems}, NIPS '23, Red Hook, NY, USA, 2024. Curran Associates Inc.

\bibitem[Alman \& Williams(2015)Alman and Williams]{AW15}
Josh Alman and Ryan Williams.
\newblock Probabilistic polynomials and hamming nearest neighbors.
\newblock In \emph{2015 IEEE 56th Annual Symposium on Foundations of Computer Science}, pp.\  136--150. IEEE, 2015.

\bibitem[Alman et~al.(2024)Alman, Turok, Yu, and Zhang]{ATYZ24}
Josh Alman, Ethan Turok, Hantao Yu, and Hengzhi Zhang.
\newblock Tensor ranks and the fine-grained complexity of dynamic programming.
\newblock In Venkatesan Guruswami (ed.), \emph{15th Innovations in Theoretical Computer Science Conference, {ITCS} 2024, January 30 to February 2, 2024, Berkeley, CA, {USA}}, volume 287 of \emph{LIPIcs}, pp.\  4:1--4:23. Schloss Dagstuhl - Leibniz-Zentrum f{\"{u}}r Informatik, 2024.
\newblock \doi{10.4230/LIPICS.ITCS.2024.4}.
\newblock URL \url{https://doi.org/10.4230/LIPIcs.ITCS.2024.4}.

\bibitem[Andoni \& Indyk(2008)Andoni and Indyk]{AI08}
Alexandr Andoni and Piotr Indyk.
\newblock Near-optimal hashing algorithms for approximate nearest neighbor in high dimensions.
\newblock \emph{Commun. ACM}, 51\penalty0 (1):\penalty0 117–122, January 2008.
\newblock ISSN 0001-0782.
\newblock \doi{10.1145/1327452.1327494}.
\newblock URL \url{https://doi.org/10.1145/1327452.1327494}.

\bibitem[Andoni \& Razenshteyn(2015)Andoni and Razenshteyn]{AR15}
Alexandr Andoni and Ilya Razenshteyn.
\newblock Optimal data-dependent hashing for approximate near neighbors.
\newblock In \emph{Proceedings of the Forty-Seventh Annual ACM Symposium on Theory of Computing}, STOC '15, pp.\  793–801, New York, NY, USA, 2015. Association for Computing Machinery.
\newblock ISBN 9781450335362.
\newblock \doi{10.1145/2746539.2746553}.
\newblock URL \url{https://doi.org/10.1145/2746539.2746553}.

\bibitem[Andoni et~al.(2017)Andoni, Laarhoven, Razenshteyn, and Waingarten]{ALRW17}
Alexandr Andoni, Thijs Laarhoven, Ilya Razenshteyn, and Erik Waingarten.
\newblock Optimal hashing-based time-space trade-offs for approximate near neighbors.
\newblock In \emph{Proceedings of the Twenty-Eighth Annual ACM-SIAM Symposium on Discrete Algorithms}, SODA '17, pp.\  47–66, USA, 2017. Society for Industrial and Applied Mathematics.

\bibitem[Baba et~al.(2017)Baba, Nakatoh, and Minami]{BNM17}
Kensuke Baba, Tetsuya Nakatoh, and Toshiro Minami.
\newblock Plagiarism detection using document similarity based on distributed representation.
\newblock \emph{Procedia Computer Science}, 111:\penalty0 382--387, 2017.
\newblock ISSN 1877-0509.
\newblock \doi{https://doi.org/10.1016/j.procs.2017.06.038}.
\newblock URL \url{https://www.sciencedirect.com/science/article/pii/S1877050917312115}.
\newblock The 8th International Conference on Advances in Information Technology.

\bibitem[Backurs \& Indyk(2015)Backurs and Indyk]{BI15}
Arturs Backurs and Piotr Indyk.
\newblock Edit distance cannot be computed in strongly subquadratic time (unless seth is false).
\newblock In \emph{Proceedings of the Forty-Seventh Annual ACM Symposium on Theory of Computing}, STOC '15, pp.\  51–58, New York, NY, USA, 2015. Association for Computing Machinery.
\newblock ISBN 9781450335362.
\newblock \doi{10.1145/2746539.2746612}.
\newblock URL \url{https://doi.org/10.1145/2746539.2746612}.

\bibitem[Backurs et~al.(2017)Backurs, Indyk, and Schmidt]{BIS17}
Arturs Backurs, Piotr Indyk, and Ludwig Schmidt.
\newblock On the fine-grained complexity of empirical risk minimization: Kernel methods and neural networks.
\newblock \emph{Advances in Neural Information Processing Systems}, 30, 2017.

\bibitem[Bartlett et~al.(2017)Bartlett, Foster, and Telgarsky]{BFT17}
Peter~L Bartlett, Dylan~J Foster, and Matus~J Telgarsky.
\newblock Spectrally-normalized margin bounds for neural networks.
\newblock In I.~Guyon, U.~Von Luxburg, S.~Bengio, H.~Wallach, R.~Fergus, S.~Vishwanathan, and R.~Garnett (eds.), \emph{Advances in Neural Information Processing Systems}, volume~30. Curran Associates, Inc., 2017.
\newblock URL \url{https://proceedings.neurips.cc/paper_files/paper/2017/file/b22b257ad0519d4500539da3c8bcf4dd-Paper.pdf}.

\bibitem[Beltagy et~al.(2020)Beltagy, Peters, and Cohan]{BPC20}
Iz~Beltagy, Matthew~E. Peters, and Arman Cohan.
\newblock Longformer: The long-document transformer.
\newblock \emph{CoRR}, abs/2004.05150, 2020.
\newblock URL \url{https://arxiv.org/abs/2004.05150}.

\bibitem[Bhattamishra et~al.(2020)Bhattamishra, Ahuja, and Goyal]{BAG20}
Satwik Bhattamishra, Kabir Ahuja, and Navin Goyal.
\newblock On the ability and limitations of transformers to recognize formal languages.
\newblock In Bonnie Webber, Trevor Cohn, Yulan He, and Yang Liu (eds.), \emph{Proceedings of the 2020 Conference on Empirical Methods in Natural Language Processing, {EMNLP} 2020, Online, November 16-20, 2020}, pp.\  7096--7116. Association for Computational Linguistics, 2020.
\newblock \doi{10.18653/V1/2020.EMNLP-MAIN.576}.
\newblock URL \url{https://doi.org/10.18653/v1/2020.emnlp-main.576}.

\bibitem[Blei et~al.(2003)Blei, Ng, and Jordan]{BNJ03}
David~M. Blei, Andrew~Y. Ng, and Michael~I. Jordan.
\newblock Latent dirichlet allocation.
\newblock \emph{J. Mach. Learn. Res.}, 3:\penalty0 993--1022, 2003.
\newblock URL \url{https://jmlr.org/papers/v3/blei03a.html}.

\bibitem[Brown et~al.(2020)Brown, Mann, Ryder, Subbiah, Kaplan, Dhariwal, Neelakantan, Shyam, Sastry, Askell, Agarwal, Herbert-Voss, Krueger, Henighan, Child, Ramesh, Ziegler, Wu, Winter, Hesse, Chen, Sigler, Litwin, Gray, Chess, Clark, Berner, McCandlish, Radford, Sutskever, and Amodei]{GPT3}
Tom Brown, Benjamin Mann, Nick Ryder, Melanie Subbiah, Jared~D Kaplan, Prafulla Dhariwal, Arvind Neelakantan, Pranav Shyam, Girish Sastry, Amanda Askell, Sandhini Agarwal, Ariel Herbert-Voss, Gretchen Krueger, Tom Henighan, Rewon Child, Aditya Ramesh, Daniel Ziegler, Jeffrey Wu, Clemens Winter, Chris Hesse, Mark Chen, Eric Sigler, Mateusz Litwin, Scott Gray, Benjamin Chess, Jack Clark, Christopher Berner, Sam McCandlish, Alec Radford, Ilya Sutskever, and Dario Amodei.
\newblock Language models are few-shot learners.
\newblock In H.~Larochelle, M.~Ranzato, R.~Hadsell, M.F. Balcan, and H.~Lin (eds.), \emph{Advances in Neural Information Processing Systems}, volume~33, pp.\  1877--1901. Curran Associates, Inc., 2020.
\newblock URL \url{https://proceedings.neurips.cc/paper_files/paper/2020/file/1457c0d6bfcb4967418bfb8ac142f64a-Paper.pdf}.

\bibitem[Carion et~al.(2020)Carion, Massa, Synnaeve, Usunier, Kirillov, and Zagoruyko]{CMSUK20}
Nicolas Carion, Francisco Massa, Gabriel Synnaeve, Nicolas Usunier, Alexander Kirillov, and Sergey Zagoruyko.
\newblock End-to-end object detection with transformers.
\newblock In Andrea Vedaldi, Horst Bischof, Thomas Brox, and Jan{-}Michael Frahm (eds.), \emph{Computer Vision - {ECCV} 2020 - 16th European Conference, Glasgow, UK, August 23-28, 2020, Proceedings, Part {I}}, volume 12346 of \emph{Lecture Notes in Computer Science}, pp.\  213--229. Springer, 2020.
\newblock \doi{10.1007/978-3-030-58452-8\_13}.
\newblock URL \url{https://doi.org/10.1007/978-3-030-58452-8\_13}.

\bibitem[Chan \& Williams(2021)Chan and Williams]{CW21}
Timothy~M. Chan and R.~Ryan Williams.
\newblock Deterministic apsp, orthogonal vectors, and more: Quickly derandomizing razborov-smolensky.
\newblock \emph{ACM Trans. Algorithms}, 17\penalty0 (1), dec 2021.
\newblock ISSN 1549-6325.
\newblock \doi{10.1145/3402926}.
\newblock URL \url{https://doi.org/10.1145/3402926}.

\bibitem[Chen(2020)]{Chen18}
Lijie Chen.
\newblock On the hardness of approximate and exact (bichromatic) maximum inner product.
\newblock \emph{Theory Of Computing}, 16\penalty0 (4):\penalty0 1--50, 2020.

\bibitem[Chiang \& Cholak(2022)Chiang and Cholak]{CC22}
David Chiang and Peter Cholak.
\newblock Overcoming a theoretical limitation of self-attention.
\newblock In \emph{Proceedings of the 60th Annual Meeting of the Association for Computational Linguistics (Volume 1: Long Papers)}, pp.\  7654--7664, 2022.

\bibitem[Choromanski et~al.(2021)Choromanski, Likhosherstov, Dohan, Song, Gane, Sarl{\'{o}}s, Hawkins, Davis, Mohiuddin, Kaiser, Belanger, Colwell, and Weller]{CLD+21}
Krzysztof~Marcin Choromanski, Valerii Likhosherstov, David Dohan, Xingyou Song, Andreea Gane, Tam{\'{a}}s Sarl{\'{o}}s, Peter Hawkins, Jared~Quincy Davis, Afroz Mohiuddin, Lukasz Kaiser, David~Benjamin Belanger, Lucy~J. Colwell, and Adrian Weller.
\newblock Rethinking attention with performers.
\newblock In \emph{9th International Conference on Learning Representations, {ICLR} 2021, Virtual Event, Austria, May 3-7, 2021}. OpenReview.net, 2021.
\newblock URL \url{https://openreview.net/forum?id=Ua6zuk0WRH}.

\bibitem[Daras et~al.(2020)Daras, Kitaev, Odena, and Dimakis]{DKOD20}
Giannis Daras, Nikita Kitaev, Augustus Odena, and Alexandros~G. Dimakis.
\newblock {SMYRF} - efficient attention using asymmetric clustering.
\newblock In Hugo Larochelle, Marc'Aurelio Ranzato, Raia Hadsell, Maria{-}Florina Balcan, and Hsuan{-}Tien Lin (eds.), \emph{Advances in Neural Information Processing Systems 33: Annual Conference on Neural Information Processing Systems 2020, NeurIPS 2020, December 6-12, 2020, virtual}, 2020.
\newblock URL \url{https://proceedings.neurips.cc/paper/2020/hash/47d40767c7e9df50249ebfd9c7cfff77-Abstract.html}.

\bibitem[Devlin et~al.(2019{\natexlab{a}})Devlin, Chang, Lee, and Toutanova]{BERT}
Jacob Devlin, Ming{-}Wei Chang, Kenton Lee, and Kristina Toutanova.
\newblock {BERT:} pre-training of deep bidirectional transformers for language understanding.
\newblock In Jill Burstein, Christy Doran, and Thamar Solorio (eds.), \emph{Proceedings of the 2019 Conference of the North American Chapter of the Association for Computational Linguistics: Human Language Technologies, {NAACL-HLT} 2019, Minneapolis, MN, USA, June 2-7, 2019, Volume 1 (Long and Short Papers)}, pp.\  4171--4186. Association for Computational Linguistics, 2019{\natexlab{a}}.
\newblock \doi{10.18653/V1/N19-1423}.
\newblock URL \url{https://doi.org/10.18653/v1/n19-1423}.

\bibitem[Devlin et~al.(2019{\natexlab{b}})Devlin, Chang, Lee, and Toutanova]{DCLT19}
Jacob Devlin, Ming{-}Wei Chang, Kenton Lee, and Kristina Toutanova.
\newblock {BERT:} pre-training of deep bidirectional transformers for language understanding.
\newblock In Jill Burstein, Christy Doran, and Thamar Solorio (eds.), \emph{Proceedings of the 2019 Conference of the North American Chapter of the Association for Computational Linguistics: Human Language Technologies, {NAACL-HLT} 2019, Minneapolis, MN, USA, June 2-7, 2019, Volume 1 (Long and Short Papers)}, pp.\  4171--4186. Association for Computational Linguistics, 2019{\natexlab{b}}.
\newblock \doi{10.18653/V1/N19-1423}.
\newblock URL \url{https://doi.org/10.18653/v1/n19-1423}.

\bibitem[Dosovitskiy et~al.(2021)Dosovitskiy, Beyer, Kolesnikov, Weissenborn, Zhai, Unterthiner, Dehghani, Minderer, Heigold, Gelly, Uszkoreit, and Houlsby]{DBK+21}
Alexey Dosovitskiy, Lucas Beyer, Alexander Kolesnikov, Dirk Weissenborn, Xiaohua Zhai, Thomas Unterthiner, Mostafa Dehghani, Matthias Minderer, Georg Heigold, Sylvain Gelly, Jakob Uszkoreit, and Neil Houlsby.
\newblock An image is worth 16x16 words: Transformers for image recognition at scale.
\newblock In \emph{9th International Conference on Learning Representations, {ICLR} 2021, Virtual Event, Austria, May 3-7, 2021}. OpenReview.net, 2021.
\newblock URL \url{https://openreview.net/forum?id=YicbFdNTTy}.

\bibitem[Duman~Keles et~al.(2023)Duman~Keles, Wijewardena, and Hegde]{KWH23}
Feyza Duman~Keles, Pruthuvi~Mahesakya Wijewardena, and Chinmay Hegde.
\newblock On the computational complexity of self-attention.
\newblock In Shipra Agrawal and Francesco Orabona (eds.), \emph{Proceedings of The 34th International Conference on Algorithmic Learning Theory}, volume 201 of \emph{Proceedings of Machine Learning Research}, pp.\  597--619. PMLR, 20 Feb--23 Feb 2023.
\newblock URL \url{https://proceedings.mlr.press/v201/duman-keles23a.html}.

\bibitem[Engels et~al.(2024)Engels, Coleman, and Shrivastava]{ECS24}
Joshua Engels, Benjamin Coleman, and Anshumali Shrivastava.
\newblock Practical near neighbor search via group testing.
\newblock In \emph{Proceedings of the 35th International Conference on Neural Information Processing Systems}, NIPS '21, Red Hook, NY, USA, 2024. Curran Associates Inc.
\newblock ISBN 9781713845393.

\bibitem[Gu \& Dao(2023)Gu and Dao]{MAMBA23}
Albert Gu and Tri Dao.
\newblock Mamba: Linear-time sequence modeling with selective state spaces.
\newblock \emph{arXiv preprint arXiv:2312.00752}, 2023.

\bibitem[Gu et~al.(2019)Gu, Akoglu, and Rinaldo]{GAR19}
Xiaoyi Gu, Leman Akoglu, and Alessandro Rinaldo.
\newblock \emph{Statistical analysis of nearest neighbor methods for anomaly detection}.
\newblock Curran Associates Inc., Red Hook, NY, USA, 2019.

\bibitem[Hahn(2020)]{Hahn20}
Michael Hahn.
\newblock {Theoretical Limitations of Self-Attention in Neural Sequence Models}.
\newblock \emph{Transactions of the Association for Computational Linguistics}, 8:\penalty0 156--171, 01 2020.
\newblock ISSN 2307-387X.
\newblock \doi{10.1162/tacl_a_00306}.
\newblock URL \url{https://doi.org/10.1162/tacl\_a\_00306}.

\bibitem[Han et~al.(2024)Han, Jayaram, Karbasi, Mirrokni, Woodruff, and Zandieh]{HJKMW24}
Insu Han, Rajesh Jayaram, Amin Karbasi, Vahab Mirrokni, David~P. Woodruff, and Amir Zandieh.
\newblock Hyperattention: Long-context attention in near-linear time.
\newblock In \emph{The Twelfth International Conference on Learning Representations, {ICLR} 2024, Vienna, Austria, May 7-11, 2024}. OpenReview.net, 2024.
\newblock URL \url{https://openreview.net/forum?id=Eh0Od2BJIM}.

\bibitem[Hao et~al.(2022)Hao, Angluin, and Frank]{HAF22}
Yiding Hao, Dana Angluin, and Robert Frank.
\newblock Formal language recognition by hard attention transformers: Perspectives from circuit complexity.
\newblock \emph{Transactions of the Association for Computational Linguistics}, 10:\penalty0 800--810, 2022.
\newblock \doi{10.1162/tacl_a_00490}.
\newblock URL \url{https://aclanthology.org/2022.tacl-1.46}.

\bibitem[Harris(1954)]{Harris54}
Zellig~S. Harris.
\newblock Distributional structure.
\newblock \emph{WORD}, 10\penalty0 (2-3):\penalty0 146--162, 1954.
\newblock \doi{10.1080/00437956.1954.11659520}.
\newblock URL \url{https://doi.org/10.1080/00437956.1954.11659520}.

\bibitem[Hornik et~al.(1989)Hornik, Stinchcombe, and White]{HSW89}
Kurt Hornik, Maxwell Stinchcombe, and Halbert White.
\newblock Multilayer feedforward networks are universal approximators.
\newblock \emph{Neural Networks}, 2\penalty0 (5):\penalty0 359--366, 1989.
\newblock ISSN 0893-6080.
\newblock \doi{https://doi.org/10.1016/0893-6080(89)90020-8}.
\newblock URL \url{https://www.sciencedirect.com/science/article/pii/0893608089900208}.

\bibitem[Hu et~al.(2024{\natexlab{a}})Hu, Lin, Song, and Liu]{HLSL}
Jerry Yao-Chieh Hu, Thomas Lin, Zhao Song, and Han Liu.
\newblock On computational limits of modern hopfield models: A fine-grained complexity analysis.
\newblock In \emph{Forty-first International Conference on Machine Learning}, 2024{\natexlab{a}}.

\bibitem[Hu et~al.(2024{\natexlab{b}})Hu, Su, Kuo, Song, and Liu]{HSKSL24}
Jerry Yao-Chieh Hu, Maojiang Su, En-Jui Kuo, Zhao Song, and Han Liu.
\newblock Computational limits of low-rank adaptation (lora) for transformer-based models, 2024{\natexlab{b}}.
\newblock URL \url{https://arxiv.org/abs/2406.03136}.

\bibitem[Hu et~al.(2024{\natexlab{c}})Hu, Wu, Li, Pi, Song, and Liu]{HWLPSL24}
Jerry Yao-Chieh Hu, Weimin Wu, Zhuoru Li, Sophia Pi, Zhao Song, and Han Liu.
\newblock On statistical rates and provably efficient criteria of latent diffusion transformers (dits).
\newblock In A.~Globerson, L.~Mackey, D.~Belgrave, A.~Fan, U.~Paquet, J.~Tomczak, and C.~Zhang (eds.), \emph{Advances in Neural Information Processing Systems}, volume~37, pp.\  31562--31628. Curran Associates, Inc., 2024{\natexlab{c}}.
\newblock URL \url{https://proceedings.neurips.cc/paper_files/paper/2024/file/37f6bed18b9b404f53dcaec4607c4fb7-Paper-Conference.pdf}.

\bibitem[Impagliazzo \& Paturi(2001)Impagliazzo and Paturi]{IP01}
Russell Impagliazzo and Ramamohan Paturi.
\newblock On the complexity of k-sat.
\newblock \emph{Journal of Computer and System Sciences}, 62\penalty0 (2):\penalty0 367--375, 2001.
\newblock ISSN 0022-0000.
\newblock \doi{https://doi.org/10.1006/jcss.2000.1727}.
\newblock URL \url{https://www.sciencedirect.com/science/article/pii/S0022000000917276}.

\bibitem[Indyk \& Motwani(1998)Indyk and Motwani]{IM98}
Piotr Indyk and Rajeev Motwani.
\newblock Approximate nearest neighbors: towards removing the curse of dimensionality.
\newblock In \emph{Proceedings of the Thirtieth Annual ACM Symposium on Theory of Computing}, STOC '98, pp.\  604–613, New York, NY, USA, 1998. Association for Computing Machinery.
\newblock ISBN 0897919629.
\newblock \doi{10.1145/276698.276876}.
\newblock URL \url{https://doi.org/10.1145/276698.276876}.

\bibitem[Jacot et~al.(2018)Jacot, Gabriel, and Hongler]{JGH18}
Arthur Jacot, Franck Gabriel, and Cl\'{e}ment Hongler.
\newblock Neural tangent kernel: convergence and generalization in neural networks.
\newblock In \emph{Proceedings of the 32nd International Conference on Neural Information Processing Systems}, NIPS'18, pp.\  8580–8589, Red Hook, NY, USA, 2018. Curran Associates Inc.

\bibitem[Jin et~al.(2016)Jin, Zhang, Chen, and Xia]{JZCX16}
Peng Jin, Yue Zhang, Xingyuan Chen, and Yunqing Xia.
\newblock Bag-of-embeddings for text classification.
\newblock In \emph{Proceedings of the Twenty-Fifth International Joint Conference on Artificial Intelligence}, IJCAI'16, pp.\  2824–2830. AAAI Press, 2016.
\newblock ISBN 9781577357704.

\bibitem[Juluru et~al.(2021)Juluru, Shih, Keshava~Murthy, and Elnajjar]{JSKE21}
Krishna Juluru, Hao-Hsin Shih, Krishna~Nand Keshava~Murthy, and Pierre Elnajjar.
\newblock Bag-of-words technique in natural language processing: A primer for radiologists.
\newblock \emph{RadioGraphics}, 41\penalty0 (5):\penalty0 1420--1426, 2021.
\newblock \doi{10.1148/rg.2021210025}.
\newblock URL \url{https://doi.org/10.1148/rg.2021210025}.
\newblock PMID: 34388050.

\bibitem[Karthik \& Manurangsi(2020)Karthik and Manurangsi]{KM20}
C.~S. Karthik and Pasin Manurangsi.
\newblock On closest pair in euclidean metric: Monochromatic is as hard as bichromatic.
\newblock \emph{Combinatorica}, 40\penalty0 (4):\penalty0 539–573, aug 2020.
\newblock ISSN 0209-9683.
\newblock \doi{10.1007/s00493-019-4113-1}.
\newblock URL \url{https://doi.org/10.1007/s00493-019-4113-1}.

\bibitem[Katharopoulos et~al.(2020)Katharopoulos, Vyas, Pappas, and Fleuret]{KVPF20}
Angelos Katharopoulos, Apoorv Vyas, Nikolaos Pappas, and Fran{\c{c}}ois Fleuret.
\newblock Transformers are {RNN}s: Fast autoregressive transformers with linear attention.
\newblock In Hal~Daumé III and Aarti Singh (eds.), \emph{Proceedings of the 37th International Conference on Machine Learning}, volume 119 of \emph{Proceedings of Machine Learning Research}, pp.\  5156--5165. PMLR, 13--18 Jul 2020.
\newblock URL \url{https://proceedings.mlr.press/v119/katharopoulos20a.html}.

\bibitem[Kitaev et~al.(2020)Kitaev, Kaiser, and Levskaya]{KKL20}
Nikita Kitaev, Lukasz Kaiser, and Anselm Levskaya.
\newblock Reformer: The efficient transformer.
\newblock In \emph{8th International Conference on Learning Representations, {ICLR} 2020, Addis Ababa, Ethiopia, April 26-30, 2020}. OpenReview.net, 2020.
\newblock URL \url{https://openreview.net/forum?id=rkgNKkHtvB}.

\bibitem[K{\"{u}}nnemann et~al.(2017)K{\"{u}}nnemann, Paturi, and Schneider]{KPS17}
Marvin K{\"{u}}nnemann, Ramamohan Paturi, and Stefan Schneider.
\newblock On the fine-grained complexity of one-dimensional dynamic programming.
\newblock In Ioannis Chatzigiannakis, Piotr Indyk, Fabian Kuhn, and Anca Muscholl (eds.), \emph{44th International Colloquium on Automata, Languages, and Programming, {ICALP} 2017, July 10-14, 2017, Warsaw, Poland}, volume~80 of \emph{LIPIcs}, pp.\  21:1--21:15. Schloss Dagstuhl - Leibniz-Zentrum f{\"{u}}r Informatik, 2017.
\newblock \doi{10.4230/LIPICS.ICALP.2017.21}.
\newblock URL \url{https://doi.org/10.4230/LIPIcs.ICALP.2017.21}.

\bibitem[Le \& Mikolov(2014)Le and Mikolov]{LM14}
Quoc Le and Tomas Mikolov.
\newblock Distributed representations of sentences and documents.
\newblock In \emph{Proceedings of the 31st International Conference on International Conference on Machine Learning - Volume 32}, ICML'14, pp.\  II–1188–II–1196. JMLR.org, 2014.

\bibitem[Liu et~al.(2023)Liu, Ash, Goel, Krishnamurthy, and Zhang]{LAGKZ23}
Bingbin Liu, Jordan~T. Ash, Surbhi Goel, Akshay Krishnamurthy, and Cyril Zhang.
\newblock Transformers learn shortcuts to automata.
\newblock In \emph{The Eleventh International Conference on Learning Representations, {ICLR} 2023, Kigali, Rwanda, May 1-5, 2023}. OpenReview.net, 2023.
\newblock URL \url{https://openreview.net/forum?id=De4FYqjFueZ}.

\bibitem[Maass et~al.(1994)Maass, Schnitger, and Sontag]{MSS94}
W.~Maass, G.~Schnitger, and E.~D. Sontag.
\newblock \emph{A Comparison of the Computational Power of Sigmoid and Boolean Threshold Circuits}, pp.\  127--151.
\newblock Springer US, Boston, MA, 1994.
\newblock ISBN 978-1-4615-2696-4.
\newblock \doi{10.1007/978-1-4615-2696-4_4}.
\newblock URL \url{https://doi.org/10.1007/978-1-4615-2696-4_4}.

\bibitem[Mahdi et~al.(2018)Mahdi, Ahmad, and Ismail]{MAI18}
Mohammed~Najah Mahdi, Abdul~Rahim Ahmad, and Roslan Ismail.
\newblock Similarity search techniques in exploratory search: {A} review.
\newblock In \emph{{TENCON} 2018 - 2018 {IEEE} Region 10 Conference, Jeju, South Korea, October 28-31, 2018}, pp.\  2193--2198. {IEEE}, 2018.
\newblock \doi{10.1109/TENCON.2018.8650257}.
\newblock URL \url{https://doi.org/10.1109/TENCON.2018.8650257}.

\bibitem[Merrill \& Sabharwal(2024)Merrill and Sabharwal]{MS24}
William Merrill and Ashish Sabharwal.
\newblock The expressive power of transformers with chain of thought.
\newblock In \emph{The Twelfth International Conference on Learning Representations, {ICLR} 2024, Vienna, Austria, May 7-11, 2024}. OpenReview.net, 2024.
\newblock URL \url{https://openreview.net/forum?id=NjNGlPh8Wh}.

\bibitem[Merrill et~al.(2022)Merrill, Sabharwal, and Smith]{MSS22}
William Merrill, Ashish Sabharwal, and Noah~A. Smith.
\newblock Saturated transformers are constant-depth threshold circuits.
\newblock \emph{Transactions of the Association for Computational Linguistics}, 10:\penalty0 843--856, 2022.
\newblock \doi{10.1162/tacl_a_00493}.
\newblock URL \url{https://aclanthology.org/2022.tacl-1.49}.

\bibitem[Mont\'{u}far et~al.(2014)Mont\'{u}far, Pascanu, Cho, and Bengio]{MPCB14}
Guido Mont\'{u}far, Razvan Pascanu, Kyunghyun Cho, and Yoshua Bengio.
\newblock On the number of linear regions of deep neural networks.
\newblock In \emph{Proceedings of the 27th International Conference on Neural Information Processing Systems - Volume 2}, NIPS'14, pp.\  2924–2932, Cambridge, MA, USA, 2014. MIT Press.

\bibitem[OpenAI(2023)]{GPT4}
OpenAI.
\newblock {GPT-4} technical report.
\newblock \emph{CoRR}, abs/2303.08774, 2023.
\newblock \doi{10.48550/ARXIV.2303.08774}.
\newblock URL \url{https://doi.org/10.48550/arXiv.2303.08774}.

\bibitem[Ostendorff(2020)]{Ost20}
Malte Ostendorff.
\newblock Contextual document similarity for content-based literature recommender systems.
\newblock \emph{CoRR}, abs/2008.00202, 2020.
\newblock URL \url{https://arxiv.org/abs/2008.00202}.

\bibitem[Phuong \& Hutter(2022)Phuong and Hutter]{PH22}
Mary Phuong and Marcus Hutter.
\newblock Formal algorithms for transformers, 2022.
\newblock URL \url{https://arxiv.org/abs/2207.09238}.

\bibitem[Roy et~al.(2021)Roy, Saffar, Vaswani, and Grangier]{routingtransformer21}
Aurko Roy, Mohammad Saffar, Ashish Vaswani, and David Grangier.
\newblock Efficient content-based sparse attention with routing transformers.
\newblock \emph{Transactions of the Association for Computational Linguistics}, 9:\penalty0 53--68, 2021.

\bibitem[Rubinstein(2018)]{Rub18}
Aviad Rubinstein.
\newblock Hardness of approximate nearest neighbor search.
\newblock In \emph{Proceedings of the 50th Annual ACM SIGACT Symposium on Theory of Computing}, STOC 2018, pp.\  1260–1268, New York, NY, USA, 2018. Association for Computing Machinery.
\newblock ISBN 9781450355599.
\newblock \doi{10.1145/3188745.3188916}.
\newblock URL \url{https://doi.org/10.1145/3188745.3188916}.

\bibitem[Sanchez-Gomez et~al.(2021)Sanchez-Gomez, Vega-Rodríguez, and Pérez]{SVP21}
Jesus~M. Sanchez-Gomez, Miguel~A. Vega-Rodríguez, and Carlos~J. Pérez.
\newblock The impact of term-weighting schemes and similarity measures on extractive multi-document text summarization.
\newblock \emph{Expert Systems with Applications}, 169:\penalty0 114510, 2021.
\newblock ISSN 0957-4174.
\newblock \doi{https://doi.org/10.1016/j.eswa.2020.114510}.
\newblock URL \url{https://www.sciencedirect.com/science/article/pii/S0957417420311544}.

\bibitem[Sanford et~al.(2024{\natexlab{a}})Sanford, Hsu, and Telgarsky]{SHT24}
Clayton Sanford, Daniel Hsu, and Matus Telgarsky.
\newblock Transformers, parallel computation, and logarithmic depth.
\newblock In Ruslan Salakhutdinov, Zico Kolter, Katherine Heller, Adrian Weller, Nuria Oliver, Jonathan Scarlett, and Felix Berkenkamp (eds.), \emph{Proceedings of the 41st International Conference on Machine Learning}, volume 235 of \emph{Proceedings of Machine Learning Research}, pp.\  43276--43327. PMLR, 21--27 Jul 2024{\natexlab{a}}.
\newblock URL \url{https://proceedings.mlr.press/v235/sanford24a.html}.

\bibitem[Sanford et~al.(2024{\natexlab{b}})Sanford, Hsu, and Telgarsky]{SHT24B}
Clayton Sanford, Daniel Hsu, and Matus Telgarsky.
\newblock Representational strengths and limitations of transformers.
\newblock In \emph{Proceedings of the 37th International Conference on Neural Information Processing Systems}, NIPS '23, Red Hook, NY, USA, 2024{\natexlab{b}}. Curran Associates Inc.

\bibitem[Shakhnarovich et~al.(2006)Shakhnarovich, Darrell, and Indyk]{SDI06}
Gregory Shakhnarovich, Trevor Darrell, and Piotr Indyk.
\newblock \emph{Nearest-Neighbor Methods in Learning and Vision: Theory and Practice (Neural Information Processing)}.
\newblock The MIT Press, 2006.
\newblock ISBN 026219547X.

\bibitem[Sparck~Jones(1988)]{SJ88}
Karen Sparck~Jones.
\newblock \emph{A statistical interpretation of term specificity and its application in retrieval}, pp.\  132–142.
\newblock Taylor Graham Publishing, GBR, 1988.
\newblock ISBN 0947568212.

\bibitem[Strobl et~al.(2024)Strobl, Merrill, Weiss, Chiang, and Angluin]{SMWCA24}
Lena Strobl, William Merrill, Gail Weiss, David Chiang, and Dana Angluin.
\newblock {What Formal Languages Can Transformers Express? A Survey}.
\newblock \emph{Transactions of the Association for Computational Linguistics}, 12:\penalty0 543--561, 05 2024.
\newblock ISSN 2307-387X.
\newblock \doi{10.1162/tacl_a_00663}.
\newblock URL \url{https://doi.org/10.1162/tacl\_a\_00663}.

\bibitem[Tao et~al.(2002)Tao, Papadias, and Shen]{TPS02}
Yufei Tao, Dimitris Papadias, and Qiongmao Shen.
\newblock Chapter 26 - continuous nearest neighbor search.
\newblock In Philip~A. Bernstein, Yannis~E. Ioannidis, Raghu Ramakrishnan, and Dimitris Papadias (eds.), \emph{VLDB '02: Proceedings of the 28th International Conference on Very Large Databases}, pp.\  287--298. Morgan Kaufmann, San Francisco, 2002.
\newblock ISBN 978-1-55860-869-6.
\newblock \doi{https://doi.org/10.1016/B978-155860869-6/50033-0}.
\newblock URL \url{https://www.sciencedirect.com/science/article/pii/B9781558608696500330}.

\bibitem[Tay et~al.(2021)Tay, Bahri, Metzler, Juan, Zhao, and Zheng]{Synthesizer21}
Yi~Tay, Dara Bahri, Donald Metzler, Da-Cheng Juan, Zhe Zhao, and Che Zheng.
\newblock Synthesizer: Rethinking self-attention for transformer models.
\newblock In Marina Meila and Tong Zhang (eds.), \emph{Proceedings of the 38th International Conference on Machine Learning}, volume 139 of \emph{Proceedings of Machine Learning Research}, pp.\  10183--10192. PMLR, 18--24 Jul 2021.
\newblock URL \url{https://proceedings.mlr.press/v139/tay21a.html}.

\bibitem[Uddin et~al.(2022)Uddin, Haque, Lu, Moni, and Gide]{UHLMG22}
Shahadat Uddin, Ibtisham Haque, Haohui Lu, {Mohammad Ali} Moni, and Ergun Gide.
\newblock Comparative performance analysis of k-nearest neighbour (knn) algorithm and its different variants for disease prediction.
\newblock \emph{Scientific Reports}, 12\penalty0 (1), April 2022.
\newblock ISSN 2045-2322.
\newblock \doi{10.1038/s41598-022-10358-x}.
\newblock Publisher Copyright: {\textcopyright} 2022, The Author(s).

\bibitem[{Vassilevska Williams}(2015)]{Williams15talk}
Virginia {Vassilevska Williams}.
\newblock Hardness of easy problems: Basing hardness on popular conjectures such as the strong exponential time hypothesis (invited talk).
\newblock In Thore Husfeldt and Iyad~A. Kanj (eds.), \emph{10th International Symposium on Parameterized and Exact Computation, {IPEC} 2015, September 16-18, 2015, Patras, Greece}, volume~43 of \emph{LIPIcs}, pp.\  17--29. Schloss Dagstuhl - Leibniz-Zentrum f{\"{u}}r Informatik, 2015.
\newblock \doi{10.4230/LIPICS.IPEC.2015.17}.
\newblock URL \url{https://doi.org/10.4230/LIPIcs.IPEC.2015.17}.

\bibitem[Vaswani et~al.(2017)Vaswani, Shazeer, Parmar, Uszkoreit, Jones, Gomez, Kaiser, and Polosukhin]{VSPUJGKP17}
Ashish Vaswani, Noam Shazeer, Niki Parmar, Jakob Uszkoreit, Llion Jones, Aidan~N. Gomez, \L{}ukasz Kaiser, and Illia Polosukhin.
\newblock Attention is all you need.
\newblock In \emph{Proceedings of the 31st International Conference on Neural Information Processing Systems}, NIPS'17, pp.\  6000–6010, Red Hook, NY, USA, 2017. Curran Associates Inc.
\newblock ISBN 9781510860964.

\bibitem[Wang et~al.(2020)Wang, Li, Khabsa, Fang, and Ma]{WLKFM20}
Sinong Wang, Belinda~Z. Li, Madian Khabsa, Han Fang, and Hao Ma.
\newblock Linformer: Self-attention with linear complexity.
\newblock \emph{CoRR}, abs/2006.04768, 2020.
\newblock URL \url{https://arxiv.org/abs/2006.04768}.

\bibitem[Wei et~al.(2024)Wei, Chen, and Ma]{WCM24}
Colin Wei, Yining Chen, and Tengyu Ma.
\newblock Statistically meaningful approximation: a case study on approximating turing machines with transformers.
\newblock In \emph{Proceedings of the 36th International Conference on Neural Information Processing Systems}, NIPS '22, Red Hook, NY, USA, 2024. Curran Associates Inc.
\newblock ISBN 9781713871088.

\bibitem[Williams(2005)]{Williams05}
Ryan Williams.
\newblock A new algorithm for optimal 2-constraint satisfaction and its implications.
\newblock \emph{Theoretical Computer Science}, 348\penalty0 (2):\penalty0 357--365, 2005.
\newblock ISSN 0304-3975.
\newblock Automata, Languages and Programming: Algorithms and Complexity (ICALP-A 2004).

\bibitem[Williams(2018)]{Williams18}
Virginia~Vassilevska Williams.
\newblock \emph{On Some Fine-Grained Questions in Algorithms and Complexity}, pp.\  3447--3487.
\newblock 2018.
\newblock \doi{10.1142/9789813272880_0188}.
\newblock URL \url{https://www.worldscientific.com/doi/abs/10.1142/9789813272880_0188}.

\bibitem[Williams \& Williams(2018)Williams and Williams]{WW18}
Virginia~Vassilevska Williams and R.~Ryan Williams.
\newblock Subcubic equivalences between path, matrix, and triangle problems.
\newblock \emph{J. ACM}, 65\penalty0 (5), August 2018.
\newblock ISSN 0004-5411.
\newblock \doi{10.1145/3186893}.
\newblock URL \url{https://doi.org/10.1145/3186893}.

\bibitem[Yang et~al.(2019)Yang, Dai, Yang, Carbonell, Salakhutdinov, and Le]{YDYCSL19}
Zhilin Yang, Zihang Dai, Yiming Yang, Jaime Carbonell, Ruslan Salakhutdinov, and Quoc~V. Le.
\newblock \emph{XLNet: generalized autoregressive pretraining for language understanding}.
\newblock Curran Associates Inc., Red Hook, NY, USA, 2019.

\bibitem[Yao et~al.(2021)Yao, Peng, Papadimitriou, and Narasimhan]{YPPN21}
Shunyu Yao, Binghui Peng, Christos~H. Papadimitriou, and Karthik Narasimhan.
\newblock Self-attention networks can process bounded hierarchical languages.
\newblock In Chengqing Zong, Fei Xia, Wenjie Li, and Roberto Navigli (eds.), \emph{Proceedings of the 59th Annual Meeting of the Association for Computational Linguistics and the 11th International Joint Conference on Natural Language Processing, {ACL/IJCNLP} 2021, (Volume 1: Long Papers), Virtual Event, August 1-6, 2021}, pp.\  3770--3785. Association for Computational Linguistics, 2021.
\newblock \doi{10.18653/V1/2021.ACL-LONG.292}.
\newblock URL \url{https://doi.org/10.18653/v1/2021.acl-long.292}.

\bibitem[Zandieh et~al.(2023)Zandieh, Han, Daliri, and Karbasi]{ZHDK23}
Amir Zandieh, Insu Han, Majid Daliri, and Amin Karbasi.
\newblock Kdeformer: accelerating transformers via kernel density estimation.
\newblock In \emph{Proceedings of the 40th International Conference on Machine Learning}, ICML'23. JMLR.org, 2023.

\end{thebibliography}
\bibliographystyle{iclr2025_conference}

\appendix

\section{A Detailed Introduction to Fine-Grained Complexity}
\label{sec: detailed FGC}

\subsection{Hardness Conjectures: $\SETH$ and $\OVC$}

We restate our central hardness conjecture here.

\begin{definition}[Strong Exponential Time Hypothesis ($\SETH$)]
For any $\varepsilon>0$, there exists a positive integer $k$ such that $k\SAT$ requires $\Omega(2^{(1-\varepsilon)n})$ time.
\end{definition}

$\SETH$ has been one of the biggest open problem in fine-grained complexity, and one of the reasons is that $\SETH$ being true would imply $\PP \neq \NP$. See \Cref{sec: main results} for a detailed explanation of its significance and why many people believe that it is true.

We also introduce the Orthogonal Vectors problem, which is an important problem in fine-grained complexity and will be a key intermediate problem in some of our proofs.

\begin{definition}[Orthogonal Vectors ($\OV_{n,\ell}$)]
Given binary vectors $v_1,\ldots,v_n \in \{0,1\}^\ell$, $\OV_{n,\ell}$ asks to determine if there exists a pair $i \neq j$ such that $\langle v_i,v_j\rangle = 0$.
\end{definition}

Usually researchers focus on the regime when $\ell \ll n$, and thus we will assume that $\ell \leq n^{o(1)}$ throughout this paper. The straightforward algorithm for $\OV_{n,\ell}$ runs in time $O(n^2\ell)$ by simply computing the inner products between each pair of vectors. In very low dimensions, one can employ a folklore recursive approach to obtain a $O(2^\ell+n)$ time algorithm (see~\cite{CST}). When $\ell = c\log n$ for a large constant $c$, all algorithms mentioned above require quadratic time with respect to $n$, but  \cite{ARY15,CW21} gave a slight improvement, showing that for a fixed constant $c>0$, $\OV_{n,c\log n}$ can be solved in time $n^{2-1/O(\log c)}$. This is a truly subquadratic running time for any fixed constant $c$, but becomes quadratic as $c$ grows. It is still unknown whether there exists a $O(n^{2-\varepsilon})$ time algorithm for all constant $c$, where $\varepsilon$ is an absolute constant that does not depend on $c$, and the popular $\OVC$ Conjecture states that no such algorithm exists.

\begin{conjecture}[$\OVC$]
For any $\varepsilon>0$, there exists a constant $c>0$ such that $\OV_{n,c\log n}$ cannot be solved in $O(n^{2-\varepsilon})$ time.
\end{conjecture}

\cite{Williams05} showed that assuming $\SETH$ is true, then $\OVC$ is true (the other direction is unknown). Our paper will use these two conjectures interchangeably such that our hardness results can be obtained from either conjecture.

There is also a bichromatic version of $\OV_{n,\ell}$ where one is given two sets of vectors $A = \{a_1,\ldots,a_n\}, B = \{b_1,\ldots,b_n\}$ such that $a_i,b_j \in \{0,1\}^\ell$ and wants to determine if there exists $i,j$ such that $\langle a_i,b_j\rangle = 0$. In fact, these two problems are subquadratic equivalent. 

\begin{lemma}
\label{lem: bichromatic OV equivalent to OV}
There exists an algorithm for $\OV_{n,\ell}$ that runs in time $O(n^{2-\varepsilon}\cdot \poly(\ell))$ for some $\varepsilon>0$ if and only if there exists an algorithm for bichromatic $\OV_{n,\ell}$ that runs in time $O(n^{2-\varepsilon'}\cdot \poly(\ell))$ for some $\varepsilon'>0$.
\end{lemma}
\begin{proof}
First we assume that there exists an algorithm $\mathcal{A}$ for bichromatic $\OV_{n,\ell}$ that runs in time $O(n^{2-\varepsilon}\cdot \poly(\ell))$ for some $\varepsilon>0$. Given an $\OV_{n,\ell}$ instance $V = \{v_1,\ldots,v_n\}$ with $v_i \in \{0,1\}^\ell$ for all $i$, we first check if any $v_i$ is the all zero vector (if so then it is a yes instance). Otherwise, let $A = B = V$ and run $\mathcal{A}$ on $A,B$. Notice that since there is no all zero vector, $\langle v_i,v_i\rangle \neq 0$ for all $i$, and therefore there exists $i \neq j$ such that $\langle v_i,v_j\rangle = 0$ if and only if there exists $v_i = a \in A, v_j = b \in B$ such that $\langle a,b\rangle = 0$. The total running time is $O(n^{2-\varepsilon}\cdot \poly(\ell)+n\ell) = O(n^{2-\varepsilon}\cdot \poly(\ell))$.

Now we assume that there exists an algorithm $\mathcal{A}'$ for $\OV_{n,\ell}$ that runs in time $O(n^{2-\varepsilon}\cdot \poly(\ell))$ for some $\varepsilon>0$. Let $A = \{a_1,\ldots,a_n\}, B = \{b_1,\ldots,b_n\}$ with vectors $a_i,b_j \in \{0,1\}^\ell$ be a bichromatic $\OV_{n,\ell}$ instance. For each $a_i$ we construct $a_i' = (a_i,1,0) \in \mathbb{R}^{\ell+2}$ and for each $b_j$ we construct $b_j' = (b_j,0,1) \in \mathbb{R}^{\ell+2}$ and let $A' = \{a_1'\,\ldots,a_n'\}, B' = \{b_1',\ldots,b_n'\}$. Notice that now $\langle a_i',a_j'\rangle,\langle b_i',b_j'\rangle \geq 1$ for all $i,j$. As a result, if there exists $\langle a_i,b_j\rangle = 0$, then $\langle a_i',b_j'\rangle = \langle a_i,b_j\rangle+0 = 0$. Conversely, if $\langle v,w\rangle = 0$ for any $v,w \in A'\cup B'$, then it must be the case that $v = a_i'$ and $w = b_j'$ for some $i,j$ or vice versa, which implies that $\langle a_i,b_j\rangle = 0$. Therefore, running $\mathcal{A}'$ on $A'\cup B'$ will tell us whether $A,B$ is a yes instance or not. The total running time required is $O((2n)^{2-\varepsilon}\cdot \poly(\ell+2)) = O(n^{2-\varepsilon}\poly(\ell))$.
\end{proof}

When $\ell = c\log n$ for some constant $c>0$, an algorithm running in $O(n^{2-\varepsilon}\cdot \poly(\ell))$ time for some fixed constant $\varepsilon>0$ is a truly subquadratic algorithm. Therefore, Lemma \ref{lem: bichromatic OV equivalent to OV} implies that assuming $\OVC$, there is no truly subquadratic algorithm for bichromatic $\OV_{n,\ell}$. 

\subsection{Minimum Inner Product} \label{sec:appendeixminip}

In this section we introduce the minimum inner product problem, an important problem related to similarity search.

\begin{definition}[$\MinIP$]
Given a set of binary vectors $v_1,\ldots,v_n \in \{0,1\}^\ell$, $\MinIP_{n,\ell}$ asks to find one pair of $1 \leq i,j \leq n, i \neq j$ such that $\langle v_i,v_j\rangle$ is minimum.  
\end{definition} 

The trivial algorithm for $\MinIP_{n,\ell}$ takes $O(n^2\ell)$ time by enumerating all possible pairs of inner product. Sometimes we are happy with finding a pair of vectors whose inner product is close enough to optimal, so we also introduce the approximate $\MinIP$ problem as follows.

\begin{definition}[$\gamma\text{-}\MinIP$]
Given a set of binary vectors $v_1,\ldots,v_n \in \{0,1\}^\ell$, $\gamma\text{-}\MinIP_{n,\ell}$ asks to find one pair of $1 \leq i,j \leq n, i \neq j$ such that $\langle v_i,v_j\rangle$ is a $\gamma$-approximation of the minimal inner product.
\end{definition}

It is not hard to see that $\MinIP$ and $\gamma\text{-}\MinIP$ are both at least as hard as $\OV$ for any $\gamma \geq 1$ (just find the minimum inner product and see if it is $0$, and any multiplicative approximation of $0$ must be $0$). Therefore, assuming $\OVC$, for any $\varepsilon>0$ there exists $c>0$ such that $\MinIP_{n,c\log n}$ cannot be solved in $O(n^{2-\varepsilon})$ time. 

In addition, there is a decision version of $\MinIP$, by which we denote $\MinIP_{n,\ell,t}$, where one wants to know whether there exists a pair of vectors whose inner product is at most $t$ for some $0 \leq t\leq \ell$.

\begin{definition}[$\MinIP$ decision version]
Given a set of binary vectors $v_1,\ldots,v_n \in \{0,1\}^\ell$ and $0 \leq t \leq \ell$, $\MinIP_{n,\ell,t}$ asks to determine if there exists one pair of $1 \leq i,j \leq n, i \neq j$ such that $\langle v_i,v_j\rangle \leq t$. 
\end{definition}

$\MinIP_{n,\ell}$ is trivially at least as hard as $\MinIP_{n,\ell,t}$ for any $t$ because if we can calculate the minimum, then we can decide whether it is at most $t$ or not. In addition, notice that when $t \geq \ell+1$ or $t < 0$ the problem is trivial and requires constant time to respond.

The bichromatic versions of these problems can be defined analogously: given two sets $A,B$ with vectors in $\{0,1\}^\ell$, one needs to find $i,j$ that achieves (for bichromatic $\MinIP_{n,\ell}$) or approximates (for bichromatic $\gamma\text{-}\MinIP_{n,\ell}$) the minimal $\langle a_{i},b_{j}\rangle$. One can easily obtain a truly subquadratic algorithm for all three problems above given a truly subquadratic algorithm for their bichromatic versions. 

\begin{lemma}
\label{lem: MinIP reduces to bichromatic MinIP}
Suppose there exists an algorithm $\mathcal{A}$ for bichromatic $\MinIP_{n,\ell}$ that runs in time $O(n^{2-\varepsilon}\cdot \poly(\ell))$ for any $\varepsilon>0$, then there exists an algorithm for $\MinIP_{n,\ell}$ that runs in time $O(n^{2-\varepsilon'}\cdot \poly(\ell))$ for some $\varepsilon'>0$. The same statement is true for $\gamma\text{-}\MinIP_{n,\ell}$ and $\MinIP_{n,\ell,t}$.
\end{lemma}
\begin{proof}
Given a $\MinIP_{n,\ell}$ instance $V = \{v_1,\ldots,v_n\}$, we first partition $V$ into $V_1,V_2$ of equal size and run $\mathcal{A}$ on $V_1,V_2$. This will allow us to find the minimal pair in $V_1 \times V_2$. Now we further partition $V_1$ into two sets of equal size and recurse on $V_1$ to eventually find the minimal pair in $V_1 \times V_1$. Similarly we recurse on $V_2$ to find the minimal pair in $V_2 \times V_2$. The running time of our algorithm is 
\[
\poly(\ell)\cdot \sum_{i=1}^{\log n}O\Big(\Big(\frac{n}{2^i}\Big)^{(2-\varepsilon)}\Big) = O(n^{2-\varepsilon}\cdot \poly(\ell)\cdot \log n) \leq O(n^{2-\varepsilon'}\cdot \poly(\ell))
\] for some $\varepsilon'>0$. The exact same argument holds for $\gamma\text{-}\MinIP_{n,\ell}$ as well because the minimal pair must appear in some recursion where we run our algorithm on. The argument also holds for $\MinIP_{n,\ell,t}$ because we have gone over all possible pairs of vectors.
\end{proof}

\subsection{Maximum Inner Product} \label{sec:appendeixmaxip}

In this section we introduce the maximum inner product problem and its variants. The problems are basically the same as problems in the previous section.

\begin{definition}[$\MaxIP$]
Given a set of binary vectors $v_1,\ldots,v_n \in \{0,1\}^\ell$, $\MaxIP_{n,\ell}$ asks to find one pair of $1 \leq i,j \leq n, i \neq j$ such that $\langle v_i,v_j\rangle$ is maximal.  
\end{definition}

\begin{definition}[$\gamma\text{-}\MaxIP$]
Given a set of binary vectors $v_1,\ldots,v_n \in \{0,1\}^\ell$, $\gamma\text{-}\MaxIP_{n,\ell}$ asks to find one pair of $1 \leq i,j \leq n, i \neq j$ such that $\langle v_i,v_j\rangle$ is a $\gamma$-approximation of the maximal inner product.
\end{definition}

\begin{definition}[$\MaxIP$ decision version]
Given a set of binary vectors $v_1,\ldots,v_n \in \{0,1\}^\ell$, $\MaxIP_{n,\ell,t}$ asks to determine if there exists one pair of $1 \leq i,j \leq n, i\neq j$ such that $\langle v_i,v_j\rangle \geq t$.
\end{definition}

The bichromatic versions of these problems can be defined analogously. Similar to $\MinIP$, one can easily obtain a truly subquadratic algorithm for all three problems above given a truly subquadratic algorithm for their bichromatic versions again using the proof of Lemma \ref{lem: MinIP reduces to bichromatic MinIP}.

\begin{lemma}
Suppose there exists an algorithm $\mathcal{A}$ for bichromatic $(\gamma\text{-})\MaxIP_{n,\ell}$ that runs in time $O(n^{2-\varepsilon}\cdot \poly(\ell))$ for any $\varepsilon>0$, then there exists an algorithm for $(\gamma\text{-})\MaxIP_{n,\ell}$ that runs in time $O(n^{2-\varepsilon'}\cdot \poly(\ell))$ for some $\varepsilon'>0$.
\end{lemma}

It is less obvious whether $\MaxIP$ is a harder problem than $\OV$ or not. The answer is positive, and for bichromatic $\MaxIP$, there exists a simple proof.

\begin{lemma}
\label{lem: OV reduces to bichromatic MaxIP}
Suppose there exists an algorithm $\mathcal{A}$ for bichromatic $\MaxIP_{n,\ell}$ that runs in time $O(n^{2-\varepsilon}\cdot \poly(\ell))$ for any $\varepsilon>0$, then there exists an algorithm for $\OV_{n,\ell}$ that runs in time $O(n^{2-\varepsilon}\cdot \poly(\ell))$.
\end{lemma}
\begin{proof}
Given an $\OV_{n,\ell}$ instance $v_1,\ldots,v_n \in \{0,1\}^\ell$, we partition all $v_i$ into subsets $S_1,\ldots,S_\ell$ such that each $S_j$ contains all vectors with $j$ ones. Now for each pair of $1 \leq i,j \leq \ell$, let $\Bar{S}_j$ consists of vectors in $S_j$ but with all the entries flipped. As a result, for any $v \in S_i, w \in \Bar{S}_j$ we have $\langle v,w\rangle = i-\langle v,\Bar{w}\rangle$, which implies that $\langle v,\Bar{w}\rangle = 0$ if and only if $\langle v,w\rangle = i$. Therefore, running $\mathcal{A}$ on $S_i,\Bar{S}_j$ will tell us whether there is an orthogonal pair of vectors in $S_i$ and $S_j$. The total running time is $O(\ell^2\cdot n^{2-\varepsilon}\cdot \poly(\ell)) = O(n^{2-\varepsilon})\cdot \poly(\ell)$.
\end{proof}

In fact, \cite{KM20} proved a stronger statement which says that even approximate $\MaxIP_{n,\ell}$ is stronger than $\OV_{n,\ell}$ for some approximation factor.

\subsection{Least Similar Documents} \label{sec:appendixLSD}

We now formally define $\LSD$ variants, which are defined similarly to $\MinIP$ variants. Recall that $\MSD$ and variants were defined in \Cref{sec: similar documents} above.

\begin{definition}[$\LSD$]
Given $n$ document embeddings $v_1,\ldots,v_n \in \{0,1\}^\ell$, $\LSD_{n,\ell}$ asks to find $1 \leq i,j \leq n, i \neq j$ such that $\frac{\langle v_{i},v_{j}\rangle}{\|v_i\|\cdot \|v_j\|}$ is the minimum.
\end{definition}

\begin{definition}[$\gamma\text{-}\LSD$]
Given $n$ document embeddings $v_1,\ldots,v_n\in \{0,1\}^\ell$, $\gamma\text{-}\MSD_{n,\ell}$ asks to find $1 \leq i^{*},j^{*} \leq n, i^{*} \neq j^{*}$ such that
\[
\min_{1 \leq i,j \leq n}\frac{\langle v_{i},v_{j}\rangle}{\|v_i\|\cdot \|v_j\|} \leq \frac{\langle v_{i^{*}},v_{j^{*}}\rangle}{\|v_{i^{*}}\|\cdot \|v_{j^{*}}\|} \leq \gamma\cdot \min_{1 \leq i,j \leq n}\frac{\langle v_{i},v_{j}\rangle}{\|v_i\|\cdot \|v_j\|}.
\]
\end{definition}

\begin{definition}[$\LSD$ decision version]
Given $n$ document embeddings $v_1,\ldots,v_n \in \{0,1\}^\ell$, $\LSD_{n,\ell,t}$ asks to determine if there exists $1 \leq i,j \leq n, i \neq j$ such that $\frac{\langle v_{i},v_{j}\rangle}{\|v_i\|\cdot \|v_j\|} \leq t$.
\end{definition}

$\gamma\text{-}\LSD_{n,\ell}$ and $\LSD_{n,\ell,t}$ are both easier than $\LSD$. In addition, the existence of truly subquadratic time algorithm for $\LSD_{n,\ell,t}$ for all $t$ again implies a truly subquadratic time algorithm for $\LSD_{n,\ell}$ using binary search. Bichromatic versions of these problems can be defined analogously and the proof of Lemma \ref{lem: MinIP reduces to bichromatic MinIP} again tells us that bichromatic versions are harder.

\section{Proof of \Cref{thm: approximate MSD is hard}}

\label{sec: normalized MaxIP is hard}

In this section we prove \Cref{thm: approximate MSD is hard}, and our ideas are similar to the ideas in section 5 and 6 of \cite{KM20}. Recall the theorem as follows.
\begin{theorem}[\Cref{thm: approximate MSD is hard}]
Assuming $\SETH$ or $\OVC$ and $\gamma \leq (1+\frac{1}{\log\log n})^{\frac{\log n}{(\log\log n)^2}} = 2^{(\log n)^{1-o(1)}}$, for every $\varepsilon>0$ there exists a constant $c>0$ such that there is no algorithm for $\gamma\text{-}\MSD_{n,\ell}$ that runs in time $O(n^{2-\varepsilon})$ where $\ell = (\log n)^{\frac{c\log n}{(\log\log n)^2}}$.
\end{theorem}

We break down the proof into several lemmas below.

\begin{lemma}
\label{lem: KM20 Theorem 6.1}
Suppose there exists an algorithm for $\gamma\text{-}\MSD_{n,\ell}$ where $\ell = (\log n)^{\frac{c\log n}{(\log\log n)^2}}$ for any constant $c>0$, $\gamma \leq (1+\frac{1}{\log\log n})^{\frac{\log n}{(\log\log n)^2}}$ that runs in $O(n^{2-\varepsilon})$ time for any $\varepsilon>0$, then there exists an algorithm for $(1+\frac{1}{\log\log n})\text{-}\MSD_{n,(\log n)^k}$ with running time $O(n^{2-\varepsilon})$ for any constant $k>0$.
\end{lemma}
\begin{proof}
Let $v_1,\ldots,v_n \in \{0,1\}^{(\log n)^k}$ be an instance of $(1+\frac{1}{\log\log n})\text{-}\MSD_{n,(\log n)^k}$ for any constant $k>0$. Construct $V' = \{v_1^{\otimes q},\ldots,v_n^{\otimes q}\}$ for $q = \frac{\log n}{(\log\log n)^2}$: all the vectors in $V'$ have dimension $(\log n)^{kq} = (\log n)^{\frac{k\log n}{(\log\log n)^2}}$, so we can run the algorithm provided to find a $\gamma$-approximation of $\MSD$ on $V'$ for some $\gamma \leq (1+\frac{1}{\log\log n})^{q}$. Notice that we have 
\[
\frac{\langle v_i^{\otimes q},v_j^{\otimes q}\rangle}{\|v_i^{\otimes q}\|\cdot \|v_j^{\otimes q}\|} = \Big(\frac{\langle v_i,v_j\rangle}{\|v_i\|\cdot \|v_j\|}\Big)^q
\] for any $i,j$, so an $\gamma$-approximation of $\MSD$ on $V'$ is a $\gamma^{1/q} = (1+\frac{1}{\log\log n})$ approximation of $\MSD$ on $v_1,\ldots,v_n$. 
\end{proof}

Now we define the bichromatic $\gamma\text{-Additive-}\MSD$ problem as: given a $A,B \subseteq \{0,1\}^\ell, \alpha \in [0,1]$ with $|A| = |B| = n$, we want to distinguish between the following two cases:
\begin{enumerate}
    \item Yes instance: There exists $(a,b) \in A\times B$ such that $\frac{\langle a,b\rangle}{\|a\|\cdot \|b\|} \geq \alpha$.
    \item No instance: For every $(a,b) \in A\times B$ we have $\frac{\langle a,b\rangle}{\|a\|\cdot \|b\|}<\alpha-\gamma$.
\end{enumerate}

\begin{lemma}
\label{lem: KM20 Theorem 6.2}
Suppose there exists an algorithm $\mathcal{A}$ for $(1+\frac{1}{\log\log n})\text{-}\MSD_{n,(\log n)^k}$ for any constant $k>0$ with $O(n^{2-\varepsilon})$ running time, then there exists an algorithm for bichromatic $\frac{\log n}{\ell}\textup{-Additive-}\MSD_{n,c\log n}$ for any $c>0$ in time $O(n^{2-\varepsilon'})$ for some $\varepsilon'>0$.
\end{lemma}
\begin{proof}
The proof follows from \Cref{thm: KM20 Theorem 6.2 real} and \Cref{lem: additive MaxIP to additive MSD}
\end{proof}

\begin{theorem}[\cite{KM20} Theorem 6.2]
\label{thm: KM20 Theorem 6.2 real}
Suppose there exists an algorithm for $(1+\frac{1}{\log\log n})\textup{-}\MSD_{n,(\log n)^k}$ for any constant $k>0$ that runs in $O(n^{2-\varepsilon})$ time for any $\varepsilon>0$, then there exists an algorithm for bichromatic $(\log n)\textup{-Additive-}\MaxIP_{n,c\log n}$ for any constant $c>0$ that runs in $O(n^{2-\varepsilon'})$ time for some $\varepsilon'>0$.
\end{theorem}

\begin{lemma}
\label{lem: additive MaxIP to additive MSD}
Bichromatic $\frac{\log n}{\ell}\textup{-Additive-}\MSD_{n,\ell}$ with $\ell = O(\log n)$ and bichromatic $(\log n)\textup{-Additive-}\MaxIP_{n,\ell'}$ with $\ell' = O(\log n)$ are subquadratic equivalent.
\end{lemma}
\begin{proof}
Given a bichromatic $(\log n)\textup{-Additive-}\MaxIP_{n,\ell}$ instance with sets $A = \{a_1,\ldots,a_n\},B = \{b_1,\ldots,b_n\}$ and integer $\alpha$, we construct $A' = \{a_1',\ldots,a_n'\},B' = \{b_1',\ldots,b_n'\}$ as follows: for each $a_i$ we first attach $\ell-\|a_i\|_1$ many ones at the end and another $\ell+\|a_i\|_1$ zeros to obtain $a_i' \in \{0,1\}^{3\ell}$, and for each $b_j$ we first attach $\ell+\|b_j\|_1$ zeros at the end and another $\ell-\|b_j\|_1$ ones to obtain $b_j' \in \{0,1\}^{3\ell}$. Now all $a_i',b_j'$ have $\ell$ many ones, and $\langle a_i',b_j'\rangle = \langle a_i,b_j\rangle$ for all $i,j$ by our construction. Therefore, running our algorithm for $\frac{\log n}{\ell}\textup{-Additive-}\MSD_{n,\ell}$ on $A',B'$ and $\frac{\alpha}{\ell}$ will tell us whether: 
\begin{enumerate}
    \item there exists $(a',b') \in A'\times B'$ such that $\frac{\langle a',b'\rangle}{\|a'\|\cdot \|b'\|} \geq \frac{\alpha}{\ell}$, which is equivalent to $\langle a,b\rangle = \langle a',b'\rangle \geq \alpha$, or
    \item for every $(a',b') \in A'\times B'$ we have $\frac{\langle a',b'\rangle}{\|a'\|\cdot \|b'\|} < \frac{\alpha}{\ell}-\frac{\log n}{\ell}$, which is equivalent to $\langle a,b\rangle = \langle a',b'\rangle <\alpha-\log n$ for all $a \in A,b \in B$.
\end{enumerate} The running time of our algorithm is $O(n^{2-\varepsilon}\cdot \poly(3\ell)) = O(n^{2-\varepsilon}\cdot \poly(\ell))$.

The reduction for the other direction is similar. Suppose we are given a bichromatic $\frac{\log n}{\ell}\textup{-Additive-}\MSD_{n,\ell}$ instance with sets $A = \{a_1,\ldots,a_n\}, B = \{b_1,\ldots,b_n\}$ and $\alpha \in [0,1]$, we construct the exact same $A',B'$ as before such that $\langle a_i',b_j'\rangle = \langle a_i,b_j\rangle$ for all $i,j$. Since $\|a_i'\| = \|b_j'\| = \sqrt{\ell}$ for all $i,j$ now, the same argument implies that running the algorithm on $A',B'$ and $\alpha\cdot \ell$ will solve the problem.
\end{proof}

\begin{lemma}[\cite{Chen18}]
\label{lem: Chen18 Lemma 3.7}
Suppose there exists an algorithm for $(\log n)\textup{-Additive-}\BMaxIP_{n,c\log n}$ for any constant $c>0$ running in time $O(n^{2-\varepsilon'})$ for some $\varepsilon'>0$, then there exists an algorithm for $\OV_{n,c'\log n}$ for any constant $c'>0$ with running time $O(n^{2-\frac{\varepsilon'}{2}})$, thus refuting $\SETH$ and $\OVC$.
\end{lemma}

\begin{proof}[Proof of \Cref{thm: approximate MSD is hard}]
This follows from a combination of Lemma ~\ref{lem: KM20 Theorem 6.1}, Lemma ~\ref{lem: KM20 Theorem 6.2}, Lemma ~\ref{lem: additive MaxIP to additive MSD} and Lemma ~\ref{lem: Chen18 Lemma 3.7}.
\end{proof} 

\section{Transformers can solve $\MaxIP,\MinIP,\MSD_{n,\ell,t}$ and $\LSD_{n,\ell,t}$}
\label{sec: representational stength proofs}

\begin{theorem}
\label{thm: transformers can solve MaxIP and MinIP}
A attention unit with input and output MLPs with parameters $d = \ell, d_{\inn} = \ell+1, d_{\out} = 1, m \geq \ell+1$ can solve $\MaxIP_{n,\ell,t}$ and $\MinIP_{n,\ell,t}$ for $1 \leq t \leq \ell$.
\end{theorem}
\begin{proof}
Given a $\MaxIP_{n,\ell}$ instance $v_1,\ldots,v_n \in \{0,1\}^\ell$ and $V \in \mathbb{R}^{n\times \ell}$ such that $V_{i,:} = v_i$ for all $i$, let $x_i = (v_i,1) \in \{0,1\}^{\ell+1}$ for all $i$, $x_{n+1}:= (0,\ldots,0,t+1) \in \mathbb{R}^{\ell+1}$ and $X \in \mathbb{R}^{(n+1) \times (\ell+1)}$ be such that $X_{i,:} = x_i$ for all $i$. Let $Q,K \in \mathbb{R}^{d_{\inn}\times m}$ be such that $QK^{\top} = 3\log n\cdot I_{d_{\inn}}$, $V = (1,1,\ldots,1,0) \in \mathbb{R}^{\ell+1}$ and $A_{Q,K,V}$ denote this attention head. 

We want to send $X$ to the attention head, which could be done by many ways given the input $V$. For example, we can add a ``end token" \footnote{\cite{SHT24B} has a similar assumption.} to the $n$ documents that is always embedded into the vector $(0,\ldots,0,t+1)\in \mathbb{R}^{\ell}$. Then we can use a MLP to send $V$ to $X$ (see Lemma \ref{lem: first MLP} for a formal proof). Now we can check that the $i$-th entry of $A_{Q,K,V}(X)$ is 
\[
\sum_{j=1}^{n}\frac{\exp(3\log n\cdot \langle x_i,x_j\rangle)}{\sum_{k=1}^{n}\exp(3\log n\cdot \langle x_i,x_k\rangle)+\exp(3(t+1)\log n)}\cdot \|v_j\|_1 = \frac{\sum_{j=1}^{n}n^{3\langle x_i,x_j\rangle}\cdot \|v_j\|_1}{\sum_{k=1}^{n}n^{3\langle x_i,x_k\rangle}+n^{3(t+1)}}.
\] Let $1 \leq i^{*} \neq j^{*} \leq n$ be such that $\langle v_{i^{*}},v_{j^{*}}\rangle$ is the largest. Therefore, if $\langle v_{i^{*}},v_{j^{*}}\rangle \geq t$, then $\langle x_{i^{*}},x_{j^{*}}\rangle \geq t+1$, which means that 
\[
\frac{\sum_{j=1}^{n}n^{3\langle x_{i^{*}},x_j\rangle}\cdot \|x_j\|_1}{\sum_{k=1}^{n}n^{3\langle x_{i^{*}},x_k\rangle}+n^{3(t+1)}} \geq \frac{n^{3\langle x_{i^{*}},x_{j^{*}}\rangle}\cdot \|v_{j^{*}}\|_1}{\sum_{k=1}^{n}n^{3\langle x_{i^{*}},x_{k}\rangle}+n^{3(t+1)}} \geq \frac{\|v_{j^{*}}\|_1}{n+1} \geq \frac{1}{n+1}.
\] On the other hand, if $\langle v_{i^{*}},v_{j^{*}}\rangle \leq t-1$, then $\langle x_{i^{*}},x_{j^{*}}\rangle \leq t$, which means 
\[
\frac{\sum_{j=1}^{n}n^{3\langle x_{i^{*}},x_j\rangle}\cdot \|v_j\|_1}{\sum_{k=1}^{n}n^{3\langle x_{i^{*}},x_k\rangle}+n^{3(t+1)}} \leq \ell\cdot \frac{n\cdot n^{3t}}{n^{3(t+1)}} = \frac{\ell}{n^2} < \frac{1}{(n+1)^{1.5}}.
\] Therefore, we can use the second MLP $\varphi_2$ to map $A_{Q,K,V}(X)$ to $1$ if any of its entry is at least $\frac{1}{n+1}$ and $0$ if all its entries are at most $\frac{1}{(n+1)^{1.5}}$ (see Lemma \ref{lem: second MLP} for a formal proof of existence).

The proof for $\MinIP$ is almost the same, except that we new let $QK^{\top} = -3\log n\cdot I_{d_{\inn}}$ and $x_{n+1} = [0,\ldots,0,-(t+1)] \in \mathbb{R}^{\ell+1}$ instead. Now the $i$-th entry of $A_{Q,K,V}(X)$ is 
\[
\frac{\sum_{j=1}^{n}n^{-3\langle x_i,x_j\rangle}\cdot \|v_j\|_1}{\sum_{k=1}^{n}n^{-3\langle x_i,x_k\rangle}+n^{-3(t+1)}}.
\] Let $i^{*},j^{*}$ be such that $\langle x_{i^{*}},x_{j^{*}} \rangle$ is the smallest. Therefore, if $\langle v_{i^{*}},v_{j^{*}} \rangle \leq t$, then $\langle x_{i^{*}},x_{j^{*}}\rangle \leq t+1$, which means that 
\[
\frac{\sum_{j=1}^{n}n^{-3\langle x_{i^{*}},x_j\rangle}\cdot \|v_j\|_1}{\sum_{k=1}^{n}n^{-3\langle x_{i^{*}},x_k\rangle}+n^{-3(t+1)}} \geq \frac{n^{-3\langle x_{i^{*}},x_{j^{*}}\rangle}}{\sum_{k=1}^{n}n^{-3\langle x_{i^{*}},x_k\rangle}+n^{-3(t+1)}} \geq \frac{1}{n+1}.
\] On the other hand, if $\langle v_{i^{*}},v_{j^{*}}\rangle \geq t+1$, then $\langle x_{i^{*}},x_{j^{*}}\rangle \geq t+2$, which means 
\[
\frac{\sum_{j=1}^{n}n^{-3\langle x_{i^{*}},x_j\rangle}\cdot \|v_j\|_1}{\sum_{k=1}^{n}n^{-3\langle x_{i^{*}},x_k\rangle}+n^{-3(t+1)}} \leq \ell\cdot \frac{n\cdot n^{-3(t+2)}}{n^{-3(t+1)}} = \frac{\ell}{n^2} < \frac{1}{(n+1)^{1.5}}.
\]
\end{proof} 

\begin{theorem}
\label{thm: transformers can solve SD}
A attention unit with input and output MLPs with parameters $d = \ell, d_{\inn} = \ell+1, d_{\out} = 1, m \geq \ell+1$ can solve $\MSD_{n,\ell,t}$ and $\LSD_{n,\ell,t}$ for any $t \in [0,1]$.
\end{theorem}
\begin{proof}
When $t = 0$ the problem is trivial so we simply need to output $1$ using MLPs, so without losing of generality we assume $t \neq 0$. Given a $\MSD_{n,\ell}$ instance $v_1,\ldots,v_n \in \{0,1\}^\ell$ and $V \in \mathbb{R}^{n \times \ell}$ such that $V_{i,:} = v_i$ for all $i$, let $x_i = (\frac{v_i}{\|v_i\|_1},1) \in \mathbb{R}^{\ell+1}$ for all $i$, $x_{n+1} := (0,\ldots,0,t+1) \in \mathbb{R}^{\ell+1}$ and $X \in \mathbb{R}^{(n+1) \times (\ell+1)}$ be such that $X_{i,:} = x_i$ for all $i$. Let $Q,K \in \mathbb{R}^{d_{\inn} \times m}$ be such that $QK^{\top} = 3\log n\cdot I_{d_{\inn}}, V = (1,1,\ldots,1,0)\in \mathbb{R}^{\ell+1}$ and $A_{Q,K,V}$ denote this attention head.

Since $\MSD$ is exactly $\MaxIP$ after we normalize the document embeddings, the proof of \Cref{thm: transformers can solve MaxIP and MinIP} implies that it suffices to send $X$ to the attention head. See Lemma \ref{lem: first MLP for MSD and LSD} for a construction.

The proof for $\LSD$ is exactly the same as the proof for $\MinIP$ after applying the $\varphi_1$ construced in Lemma \ref{lem: first MLP for MSD and LSD}.
\end{proof}

\section{MLP Constructions}

\begin{lemma}
\label{lem: second MLP}
For any $a,b \in \mathbb{R}$ such that $b>a$, there exists a continuous function $f: \mathbb{R}^\ell \rightarrow \mathbb{R}$ such that 
\[
f(x) = 
\begin{cases}
    1 &\textup{if } x[i] \geq b \textup{ for any } 1 \leq i \leq \ell\\
    0 &\textup{if }  x[i] < a \ \forall 1 \leq i \leq \ell.
\end{cases}
\]
\end{lemma}
\begin{proof}
Firstly we define $g: \mathbb{R}$ such that 
\[
g(x) = 
\begin{cases}
    1 &\textup{if } x \geq b\\
    \frac{1}{b-a}(x-a) &\textup{if } a \leq x < b\\
    0 &\textup{if } x < a.
\end{cases}
\] Now we let 
\[
f(x) = 1-\prod_{i=1}^{\ell}(1-g(x[i])).
\] It is not hard to see that $g$ is a continuous function, and therefore $f$ is a continuous function. When $x[i] \geq b$ for any $1 \leq i \leq \ell$, $g(x[i]) = 1$ and therefore $f(x) = 1-0 = 1$. On the other hand, if $x[i]<a$ for all $i$, then $f(x) = 1-1 = 0$.
\end{proof}

\begin{lemma}
\label{lem: first MLP}
There exists a continuous function $f: \mathbb{R}^\ell \rightarrow \mathbb{R}^{\ell+1}$ such that 
\[
f(x) = 
\begin{cases}
    (x,1) &\textup{if } x[\ell] \leq 1\\
    (0,x) &\textup{otherwise}.
\end{cases}
\]
\end{lemma}
\begin{proof}
First we define a function $g:\mathbb{R} \rightarrow \mathbb{R}$ such that 
\[
g(x) = 
\begin{cases}
    1 &\textup{if } x \leq 1\\
    2-x &\textup{if } 1<x \leq 2\\
    0 &\textup{if } x > 2,
\end{cases}
\] and we also define $f_1,f_2:\mathbb{R}^\ell \rightarrow \mathbb{R}^{\ell+1}$ where
\[
f_1(x) = (x,1), f_2(x) = (0,x).
\] It is not hard to see that $g,f_1,f_2$ are all continuous functions, so we let
\[
f(x) = g(x[\ell])\cdot f_1(x)+(1-g(x[\ell]))\cdot f_2(x)
\] such that $f$ is also continuous. We can check that $f$ satisfies the requirement in the lemma statement.
\end{proof}

\begin{lemma}
\label{lem: first MLP for MSD and LSD}
There exists a continuous function $f:\mathbb{R}^{\ell} \rightarrow \mathbb{R}^{\ell+1}$ such that 
\[
f(x) = 
\begin{cases}
    (\frac{x}{\|x\|_1},1) &\textup{if } x[d] \leq 1\\
    (0,\frac{x}{\|x\|_1}) &\textup{otherwise.}
\end{cases}
\]
\end{lemma}
\begin{proof}
Let $g$ be the same functions as in Lemma \ref{lem: first MLP} and $f_1,f_2:\mathbb{R}^{\ell} \rightarrow \mathbb{R}^{\ell+1}$ where
\[
f_1(x) = \Big(\frac{x}{\|x\|_1},1\Big), f_2(x) = \Big(0,\frac{x}{\|x\|_1}\Big).
\] $f_1,f_2$ are continuous function over $\mathbb{R}^{\ell}$, and therefore we let
\[
f(x) = g(x[\ell])\cdot f_1(x)+(1-g(x[\ell]))\cdot f_2(x)
\] such that $f$ is also continuous. We can check that $f$ satisfies the requirement in the lemma statement.
\end{proof}

\end{document}